\documentclass[11pt,a4paper]{article}

\usepackage[top=37mm, bottom=33mm, left=27mm, right=27mm]{geometry}

\usepackage[T1]{fontenc}

\usepackage{amssymb,amsfonts,amsthm}
\usepackage{mathtools,bm}
\usepackage{xcolor}
\definecolor{ms-green}{HTML}{009a55}   
\definecolor{ms-blue}{HTML}{006fb8}   
\definecolor{ms-magenta}{HTML}{ec008d}   
\usepackage{adjustbox}
\usepackage{float}

\usepackage{titlesec}
\titleformat{\section}[hang]{\normalfont\scshape\centering}{\thesection.}{1em}{}
\titleformat{\subsection}[runin]
  {\normalfont\sffamily\normalsize\bfseries}{\thesubsection.}{0.5em}{\addperiod}
\titleformat{\subsubsection}[runin]
  {\normalfont\normalsize\itshape}{\thesubsubsection.}{0.5em}{\addperiod}
\newcommand{\addperiod}[1]{#1.}

\makeatletter
\def\@maketitle{%
  \newpage
  \null
  \vskip 2em%
  \begin{center}%
    {\normalfont\normalsize\bfseries\MakeUppercase{\@title} \par}%
    \vskip 1.5em%
    {\normalfont\normalsize
      \lineskip .5em%
      \begin{tabular}[t]{c}%
         \@author
      \end{tabular}\par}%
  \end{center}%
  \par
  \vskip 1.5em}
\makeatother

\usepackage{enumerate}

\numberwithin{equation}{section}

\usepackage[font={small}, textfont={sf}, labelfont={bf,sf}, margin=1cm]{caption}
\usepackage{algorithm}
\usepackage{algpseudocode}

\usepackage{natbib}
\usepackage{multirow}
\usepackage{booktabs} 
\usepackage{tabularx}

\usepackage[bookmarksopen=false, bookmarksnumbered=true, colorlinks=true, linkcolor=ms-blue, citecolor = ms-green, urlcolor = ms-magenta ]{hyperref}
\usepackage{cleveref}

\usepackage{orcidlink}

\hypersetup{
   pdftitle={Stiefel-AdamW: Geometry-Aware AdamW for Linear Factorization Blocks},
   pdfsubject={math.NA},
   pdfauthor={Emanuele Zangrando, Marco Sutti and Francesco Tudisco},
   pdfkeywords={keywords}
}

\renewenvironment{abstract}
 {\small
  \begin{center}
  \bfseries \abstractname\vspace{-.5em}\vspace{0pt}
  \end{center}
  \list{}{
    \setlength{\leftmargin}{0.0cm}%
    \setlength{\rightmargin}{\leftmargin}%
  }%
  \item\relax}
 {\endlist}

\let\Delta\varDelta
\let\Xi\varXi
\let\Pi\varPi
\let\Sigma\varSigma
\let\Omega\varOmega

\newcommand{\email}[1]{\protect\href{mailto:#1}{#1}}

\newcommand{\algname}{Stiefel-AdamW}

\newcommand{\tr}{^{\top}} 

\newcommand{\cM}{{\mathcal M}}          
\newcommand{\cMr}{\cM_{r}}              
\newcommand{\Stnr}{\mathrm{St}(n,r)}    

\newcommand{\T}{\mathrm{T}}

\newcommand{\R}{\mathbb{R}}
\newcommand{\Rmn}{\R^{m \times n}}

\newcommand{\cF}{{\mathcal F}}

\newcommand{\cL}{{\mathcal L}}
\newcommand{\cO}{{\mathcal O}}

\theoremstyle{plain}
\newtheorem{theorem}{Theorem}[section]
\newtheorem{proposition}[theorem]{Proposition}
\newtheorem{lemma}[theorem]{Lemma}

\theoremstyle{definition}
\newtheorem{definition}[theorem]{Definition}

\newtheorem{assumptions}[theorem]{Assumptions}
\crefname{assumptions}{Assumptions}{Assumptions}
\Crefname{assumptions}{Assumptions}{Assumptions}

\theoremstyle{remark}

\DeclareMathOperator{\Retraction}{Retr}

\DeclareMathOperator{\D}{D}   

\DeclareMathOperator{\rank}{rank}

\DeclareMathOperator{\Proj}{P}

\begin{document}

\title{\algname: Geometry-Aware AdamW for Linear Factorization Blocks}

\author{\MakeUppercase{Emanuele Zangrando}\thanks{Gran Sasso Science Institute, L'Aquila, Italy (\email{emanuele.zangrando@gssi.it}).}\hspace{2mm}\orcidlink{0000-0002-8410-1372}, \MakeUppercase{Marco Sutti}\thanks{Gran Sasso Science Institute, L'Aquila, Italy (\email{marco.sutti@gssi.it}). M. Sutti is a member of INdAM GNCS.}\hspace{2mm}\orcidlink{0000-0002-8410-1372} \MakeUppercase{and Francesco Tudisco$^*$}\thanks{University of Edinburgh, Edinburgh, UK (\email{f.tudisco@ed.ac.uk}).}\hspace{2mm}\orcidlink{}}


\maketitle

\begin{abstract}
A pervasive structural pattern in modern deep learning is the \emph{linear factorization block}: a submodule of the form $W = BA$ in which two parameter matrices are multiplied directly, with no intervening nonlinearity. Such blocks appear in LoRA adapters, low-rank compressed layers, query-key products of self-attention, and share a common pathology: the factorization is non-unique, which can destabilize training and limit usable learning rates. Despite this, factorization blocks are typically optimized with standard Euclidean methods that ignore the underlying geometry.
We introduce \algname{}, a near drop-in replacement for AdamW for use wherever such blocks appear. By constraining one factor on the Stiefel manifold while leaving the other Euclidean, \algname{} relaxes the full $\mathrm{GL}(\R^r)$ gauge symmetry to a compact orthogonal symmetry, ruling out factor blow-up while retaining the coordinate-wise diagonal preconditioning that gives AdamW its practical strength. Moment estimation is performed in the ambient Euclidean space, with geometry entering only through a tangent-space projection and a manifold retraction. The implementation overhead over AdamW is minimal, and we show that the resulting optimizer inherits both the stability benefits of Riemannian methods and standard convergence guarantees. We validate \algname{} on LoRA-style fine-tuning of GPT2, ViT, and Mistral 7B and on full pretraining of GPT2 on OpenWebText, showing consistent improvements over strong baselines at essentially no additional cost over AdamW.
\end{abstract}

\section{Introduction}\label{sec:introduction}
A pervasive structural pattern in modern deep-learning architectures is the presence of \emph{linear factorization blocks}: trainable submodules of the form $W = BA$, in which two parameter matrices are multiplied directly, with no intervening nonlinearity. Such blocks arise in many guises across the modern model zoo. In LoRA-style parameter-efficient fine-tuning \cite{hu2022lora,zhang2023adaptive,Hayou2024loraplus,ZhangPilanci2024,zhao2024galore,lialin2024relora,schotthofer2025geolora}, weight correctors are parametrized as $\Delta W = BA$, with $B \in \R^{m\times r}$ and $A \in \R^{r \times n}$. In low-rank pretraining \cite{MLSYS2021_94cb2887,khodak2021initialization,schotthoefer2025} and network compression \cite{Vogels:2019,saha2023matrix,mo2025parameter,Schotthoefer2022lottery}, weight matrices are similarly factorized into two trainable factors. The same structure also appears \emph{inside} standard architectures: in self-attention \cite{NIPS2017_3f5ee243}, the score matrix $W_Q W_K\tr$ factorizes through the rank-$r$ key/query head dimension, and analogous matrix-matrix factorizations appear in the recurrence kernels of several structured state-space models (SSMs). In all these cases, two parameter blocks combine multiplicatively, without an intervening nonlinearity, to produce a single effective linear operator.

These blocks share a common geometric feature: the product map $\Phi(A,B) = BA$ is highly non-unique. For any $M \in \mathrm{GL}(\R^r)$, $\Phi(A,B) = \Phi(M^{-1}A,\,BM)$, so each effective weight $W$ corresponds to a continuous family of parameter pairs. This gauge symmetry has direct consequences for optimization. If the loss $\cL$, viewed as a function of the effective matrix $W$, has a stationary point, then $\cL \circ \Phi$ has an entire orbit of equivalent stationary points parametrized by $\mathrm{GL}(\R^r)$. More worryingly, training along these blocks can be numerically unstable even without spurious minimizers: a sequence $(M_n^{-1} A_n,\, B_n M_n)$ with $M_n \to 0$ produces a bounded product $W_n$ even though one factor diverges. These issues are well documented in practice \cite{Mishra2014,Schotthoefer2022lottery,zangrando2024geometry} and become especially pronounced under unbalanced initializations, such as those typical of LoRA.

Despite the prevalence of this structure, standard practice for training linear factorization blocks is to apply Euclidean optimizers, most commonly AdamW \cite{Loshchilov:2019}, directly to the unconstrained pair $(A,B)$, ignoring the underlying geometry entirely. 

The natural geometric remedy is to optimize on the fixed-rank matrix manifold 
\[
    \cMr = \R^{r \times n}_* \times \R^{m \times r}_* / \mathrm{GL}(\R^r),
\]
viewed as a quotient \cite{Mishra2014}. Here, $\R^{m \times r}_*$ denotes the set of full-rank matrices of size $m\times r$. While geometrically clean, this approach is at odds with adaptive optimizers such as Adam, since the full $\mathrm{GL}(\R^r)$ invariance forces nonlinear adaptive updates to satisfy strong equivariance constraints that, in general, cannot be reconciled with the coordinate-wise diagonal preconditioning that gives AdamW much of its practical strength. As a result, prior Riemannian adaptive methods either give up coordinate-wise adaptivity in favor of scalar preconditioners, or modify the gradient structure rather than the preconditioner itself \cite{BecigneulGanea:2019,sakai2025a,schotthoefer2025,
bian2025finding}.

In this work, we adopt a different trade-off. Rather than enforcing full quotient invariance, we relax it to an orthogonal symmetry by restricting to the product manifold $\Stnr \times \R^{m \times r}_*$, on which $\Phi$ is
invariant only under $\mathrm{O}(r)$:
\[
    \Phi(A,B) = \Phi(Q\tr\! A, \,BQ), \quad \forall Q \in \mathrm{O}(r).
\]
This relaxation accomplishes two things at once. \emph{Geometrically}, it makes the fibers of $\Phi$ compact, which rules out the kind of factor blow-up illustrated above and stabilizes training under aggressive learning rates. \emph{Algorithmically}, because the constrained factor lives on an embedded submanifold and the other factor in a flat Euclidean space, we can perform all moment accumulation in the ambient space and only project onto the tangent space of the Stiefel manifold immediately before the update. This preserves AdamW's coordinate-wise diagonal preconditioning essentially unchanged.

The resulting algorithm, \textbf{\algname}, is best understood as a small, structural modification of AdamW. The Euclidean factor is updated by the standard AdamW step; the Stiefel factor is updated by computing the first and second moments in the ambient space, projecting the preconditioned direction onto the tangent space, and retracting back to the manifold. The retraction step is treated as a modular component: any efficient retraction on the Stiefel manifold can be used, including the Cayley transform, QR-based retraction, polar decomposition, or Newton--Schulz iteration \cite{Kovarik1970SomeIM,BjorckBowie1971,Higham:2008}. In our experiments, the choice of retraction has only a marginal effect on final performance, with the best option mildly model- and problem-dependent; see~\Cref{app:compare_retractions} and \Cref{tab:gpt_pretrain} for a comparison. The overall implementation overhead over standard AdamW is minimal, yet the resulting optimizer inherits the principal benefits of geometric optimization: numerical stability, parameter invariance to orthogonal reparametrization, low memory footprint, and provable convergence under standard assumptions.

We emphasize that \algname{} is not a method specific to LoRA or to any particular architectural family. It is a simple, near drop-in replacement for AdamW that can (and should) be used on \emph{any} factorization block of the form $W = BA$ in which the two factors are not separated by a nonlinearity, regardless of where such a block appears in the model. This includes LoRA adapters, low-rank compressed layers, and query/key projections in attention blocks.
In all these cases, applying \algname{} rather than vanilla AdamW replaces a Euclidean parametrization with hidden gauge symmetry by one in which (most of) the symmetry has been quotiented out, and does so essentially for free.

\medskip

\textbf{Contributions.} Our contributions are as follows.
\begin{enumerate}
    \item We identify factorization blocks $W = BA$ ubiquitous in modern deep-learning architectures (LoRA, low-rank compressed layers, attention, SSMs) as a unifying setting in which geometric optimization can be applied transparently, and we propose \algname{}, a near drop-in replacement for AdamW that should be used wherever such blocks appear.

    \item We propose a structured relaxation of the full $\mathrm{GL}(\R^r)$ quotient invariance to orthogonal invariance, working on the product manifold $\Stnr \times \R^{m \times r}_*$. This is precisely the relaxation that allows full coordinate-wise Adam-style adaptive preconditioning to coexist with geometry-aware updates, in contrast to prior Riemannian Adam-type methods that rely on simple scalar preconditioners.

    \item We design the algorithm so that all adaptive moment estimation is performed in ambient Euclidean space, with geometric corrections (tangent projection and retraction) applied only at the update step. As a consequence, the per-step cost and memory footprint are essentially those of AdamW.

    \item We treat the Stiefel retraction as a modular, plug-and-play component: any efficient retraction (Cayley transform, QR, polar decomposition, Newton--Schulz) can be used, with only marginal differences in practice. Approximate retractions, when second-order accurate, remain compatible with the standard convergence theory.

    \item We establish theoretical guarantees of boundedness of gradients and convergence of the regret function under standard assumptions.

    \item We validate \algname{} empirically across a range of representative tasks, from LoRA-style fine-tuning (GPT2, ViT, Mistral 7B) to full LLM pretraining (GPT2 on OpenWebText), demonstrating that the additional cost over AdamW is minimal while consistently improving stability and final performance.
\end{enumerate}

\section{Related Work}
\textbf{Riemannian Optimization.}
Riemannian gradient methods on embedded and quotient manifolds, including the Stiefel and Grassmann manifolds, have a long history in numerical optimization \cite{Luenberger:1972,Gabay:1976,EAS:1998,Yang2007}, and have since been extended to trust-region \cite{ABG:2007}, quasi-Newton \cite{RingWirth:2012}, and conjugate-gradient \cite{SatoIwai:2015,Sato:2016,Sato:2022} methods. Key tools including retraction mappings \cite{Absil:2012,AbsilOseledets2015}, efficient preconditioners \cite{Vandereycken:2010,BOUMAL2015200}, and quotient geometries \cite{Mishra2014} have been thoroughly developed; comprehensive treatments can be found in \cite{AMS:2008,Sato:2021,boumal_2023}.

\textbf{Adaptive Methods and their Riemannian Extensions.}
In deep learning, adaptive optimizers such as AdaGrad \cite{Duchi:2011}, RMSProp \cite{Hinton:2012}, Adam \cite{kingma2015adam}, AMSGrad \cite{convergence_adam_beyond}, and AdamW \cite{Loshchilov:2019}, are the de facto standard, combining fast convergence, robustness, and low memory overhead. Extending them to manifold-constrained settings is nontrivial, because adaptive moments and preconditioning must respect the underlying geometry. Stochastic Riemannian optimization began with Riemannian SGD \cite{Bonnabel:2013} and variance-reduced variants \cite{NIPS2016_98e6f172,pmlr-v80-kasai18a}. More recent work has proposed Riemannian AdaGrad and AMSGrad \cite{BecigneulGanea:2019}, modified AMSGrad schemes \cite{SakaiIiduka:2021}, RASA \cite{kasai2019riemannian}, and Riemannian adaptive gradient methods with theoretical guarantees \cite{bian2025finding,sakai2025a}. A common limitation of these approaches is that they rely on scalar or geometry-compatible preconditioners; the fully coordinate-wise diagonal preconditioning central to AdamW is generally incompatible with strict manifold invariance, which is the gap \algname{} is designed to close.

\medskip

\textbf{Linear Factorization Blocks in Deep Learning.}
Matrix factorization blocks of the form $W = BA$, in which two parameter matrices multiply directly without an intervening nonlinearity, appear throughout modern architectures. They arise explicitly in LoRA-style parameter-efficient fine-tuning \cite{hu2022lora} and its variants \cite{Hayou2024loraplus,ZhangPilanci2024,zhu2024imbalanceregularized,Wang2025lorapro,schotthofer2025geolora}, in network compression \cite{Vogels:2019,saha2023matrix,Schotthoefer2022lottery}, and implicitly inside standard architectures such as the query-key product in self-attention \cite{NIPS2017_3f5ee243}. The non-uniqueness of such factorizations and its consequences for training stability have been studied through dynamical low-rank approximation \cite{KochLubich:2007,HnatiukKusch:2026,zangrando2024geometry} and quotient-manifold optimization \cite{Mishra2014}. Methods such as GeoLoRA \cite{schotthofer2025geolora} have explicitly exploited Grassmannian geometry in low-rank adaptation, and recent work has combined momentum and adaptivity with Riemannian updates \cite{schotthoefer2025,sakai2025a}.

\medskip

\textbf{Relation to Prior Work.} 
The closest methods to \algname{} are GeoLoRA \cite{schotthofer2025geolora} and RAdam \cite{BecigneulGanea:2019} or its Stiefel-specific version, Cayley Adam \cite{LFT:2020}. GeoLoRA enforces geometry on the full factorization via a quotient-manifold formulation, which precludes coordinate-wise adaptive preconditioning. RAdam defines a Riemannian Adam on the Stiefel manifold but uses a scalar preconditioner, departing from AdamW's diagonal adaptivity. In contrast, \algname{} constrains only one factor to the Stiefel manifold while leaving the other in Euclidean space. This product-manifold structure is precisely what makes full coordinate-wise adaptive preconditioning tractable: moment accumulation is performed in the ambient space for both factors, and geometry enters only through one factor via a tangent-space projection followed by a retraction on the Stiefel manifold. The result is an optimizer that inherits the stability benefits of Riemannian methods and the practical performance of AdamW, at minimal additional cost.

\section{The Proposed Method: \algname}
\subsection{Problem Setup}

We consider a trainable linear factorization block of the form $W = BA$,
where $B \in \R^{m \times r}$ and $A \in \R^{r \times n}$. In order to avoid potential instabilities due to non-uniqueness of this representation, we impose a row-orthonormality constraint on $A$, namely $ AA\tr = I_r $, i.e., we require $A\tr$ to lie on the Stiefel manifold $ \Stnr = \{ X \in \R^{n \times r} \colon X\tr\! X = I_r \},$ while leaving $B$ unconstrained. This restriction reduces the invariance group from $\mathrm{GL}(\R^r)$ to the compact group $\mathrm{O}(r)$, making the fibers of $\Phi$ compact and ruling out the factor blow-up illustrated in \Cref{sec:introduction}; see also \Cref{sec:boundedness} for further details. The resulting product-manifold structure $\Stnr \times \R^{m \times r}$ enables an efficient fully coordinate-wise adaptive update: moment accumulation is performed in the ambient Euclidean space for both factors, with geometric corrections entering only through a tangent-space projection and a retraction onto the Stiefel manifold for the $A$ factor.

\subsection{Description of the Algorithm}
In this section, we describe one iteration of \algname; the pseudocode is given in \Cref{algo:StiefelAdamW_BA}. For simplicity of exposition, we omit bias correction and explicit weight decay; both can be incorporated straightforwardly as in AdamW.

\begin{algorithm}[t]
    \caption{Single iteration of \algname.}\label{algo:StiefelAdamW_BA}
    \begin{algorithmic}[1]
        \Require $A_{t}$ with $A_{t}A_{t}\tr = I_{r}$,\ $B_{t}$,\ 
                 $M^{A}_{t-1}$,\ $M^{B}_{t-1}$,\ $V^{A}_{t-1}$,\ $V^{B}_{t-1}$,\ 
                 $\eta_{t}$,\ $\beta_{1}$,\ $\beta_{2}$,\ $\varepsilon$
        \Statex
        \State $G_{t}^{B} \leftarrow \nabla_{B}\cL(B_{t}A_{t})\,A_{t}\tr$
               \Comment{Euclidean gradient w.r.t.\ $B$}
        \State $G_{t}^{A} \leftarrow B_{t}\tr\nabla_{A}\cL(B_{t}A_{t})$
               \Comment{Euclidean gradient w.r.t.\ $A$}
        \Statex
        \State $M^{B}_{t} \leftarrow \beta_{1}M^{B}_{t-1} + (1-\beta_{1})\,G_{t}^{B}$
               \Comment{First moment, $B$}
        \State $M^{A}_{t} \leftarrow \beta_{1}M^{A}_{t-1} + (1-\beta_{1})\,G_{t}^{A}$
               \Comment{First moment, $A$}
        \State $V^{B}_{t} \leftarrow \beta_{2}V^{B}_{t-1} + (1-\beta_{2})\,(G_{t}^{B})^{\circ 2}$
               \Comment{Second moment, $B$}
        \State $V^{A}_{t} \leftarrow \beta_{2}V^{A}_{t-1} + (1-\beta_{2})\,(G_{t}^{A})^{\circ 2}$
               \Comment{Second moment, $A$}
        \State $B_{t+1} \leftarrow B_{t} - \eta_{t} \Bigl( M^{B}_{t}/\!\left(\sqrt{V^{B}_{t}+\varepsilon}\right) + \lambda B_{t} \Bigr) $
               \Comment{Standard AdamW step on $B$}
        \State $X_{t} \leftarrow A_{t}\tr$
               \Comment{Column convention, $X_t \in \Stnr$}
        \State $D_{t} \leftarrow M^{A}_{t}/(\sqrt{V^{A}_{t}+\varepsilon})$
               \Comment{Preconditioned direction}
        \State $\xi_{t} \leftarrow {-\eta_{t}}\,\Proj_{X_{t}}(D_{t})$
               \Comment{Project onto $\mathrm{T}_{X_{t}}\Stnr$,\ 
                        $\Proj_{X}(Z)=X\,\mathrm{skew}(X\tr Z)+(I-XX\tr)Z$}
        \State $X_{t+1} \leftarrow \Retraction_{X_{t}}(\xi_{t}),\,A_{t+1} \leftarrow X_{t+1}\tr$
               \Comment{Retract to $\Stnr$; see \Cref{sec:retractions}}
    \end{algorithmic}
\end{algorithm}

Let $ W_{t} \in \R^{m \times n} $ be a weight matrix at iteration $ t $, represented by the factorization $ W_{t} = B_{t} A_{t} $, where $ A_{t} \in \R^{r \times n} $ and $ B_{t} \in \R^{m \times r} $. The objective function is evaluated as $ \cL(B_{t}A_{t}) $. Next, the algorithm computes the Euclidean gradients $ G_{t}^B = \nabla_{B} \cL(B_{t}A_{t}) $ and $ G_{t}^A = \nabla_{A} \cL(B_{t}A_{t}) $. As in Adam and AdamW, \algname{}~then computes the first and second moments for both factors, i.e.,
\[
   \begin{cases}
        M^{B}_{t} = \beta_{1} M^{B}_{t-1} + (1-\beta_{1}) \, G_{t}^{B}, \\
        M^{A}_{t} = \beta_{1} M^{A}_{t-1} + (1-\beta_{1}) \, G_{t}^{A},
   \end{cases} \qquad 
   \begin{cases}
        V^{B}_{t} = \beta_{2} V^{B}_{t-1} + (1-\beta_{2}) \, (G_{t}^{B})^{\circ 2}, \\
        V^{A}_{t} = \beta_{2} V^{A}_{t-1} + (1-\beta_{2}) \, (G_{t}^{A})^{\circ 2},
   \end{cases}
\]
where $ ^{\circ 2} $ denotes elementwise squaring.

Up to this step, both factors are treated in the same way; however, the subsequent update steps do differ. Indeed, since $ B_{t} $ is unconstrained, the algorithm performs the usual AdamW update, i.e., $B_{t+1} = B_{t} - \eta_{t} \Bigl( M^{B}_{t} / ( \sqrt{V^{B}_{t} + \varepsilon} ) + \lambda B_{t} \Bigr)$, where $ \eta_{t} $ is the learning rate, $ / $ indicates elementwise division, the square root is also meant to be performed componentwise, and the $ \varepsilon >0$ is a small constant to avoid blowup of the metric. This factor requires no Riemannian machinery, which keeps the method simple and efficient.

For the orthonormal factor, we employ a retraction-based Riemannian update. For convenience, we switch to a column-orthonormal representation by defining $X_{t} \coloneqq A_{t}\tr \in \R^{n \times r},$ so that $X_{t} \in \Stnr$, i.e., $X_{t}\tr\! X_{t} = I_{r}$.
A key feature of \algname{} is that it performs adaptive moment estimation in the ambient Euclidean space before projecting onto the tangent space. This allows us to use coordinate-wise preconditioning, as in AdamW. In contrast, many existing Riemannian adaptive methods restrict the preconditioner to be scalar or geometry-compatible to preserve invariance, thereby limiting their practical effectiveness.
More precisely, the Euclidean adaptive direction is first formed as $D_{t} = M^{A}_{t} /(\sqrt{V^{A}_{t} + \varepsilon})$, and then projected onto the tangent space $ \T_{X_{t}}\Stnr $ to obtain the direction $\xi_{t} = -\eta_{t} \Proj_{X_{t}}(D_{t})$, where $\Proj_{X_{t}} \colon \R^{n \times r} \to \T_{X_{t}}\Stnr$ is the orthogonal projection onto the tangent space to $\Stnr$ at $X_{t}$, $\Proj_{X}(Z) = X \mathrm{skew}(X\tr \! Z) + (I-XX\tr) \, Z $, with $\mathrm{skew}(M)=(M-M\tr)/2$. See~\Cref{app:stiefel_geometry} for more details on the geometry of the Stiefel manifold.

\subsection{Choice of Retraction}\label{sec:retractions}

To map a tangent vector $\xi_{t}$ back onto the manifold, we need to apply a retraction mapping, $X_{t+1} = \Retraction_{X_{t}}(\xi_{t})$.
While the exponential map provides the most geometrically accurate geodesic path, it is often computationally prohibitive for large-scale problems because it requires full eigenvalue decompositions or matrix exponentials \citep{AMS:2008}. In practice, a retraction is any mapping that agrees to first order with the exponential map (i.e., is centered at the point and has the differential at the origin equal to the identity map). 

For the Stiefel manifold $\Stnr$, several efficient retractions exist with a computational complexity of $O(nr^2 + r^3)$, which is ideal for settings where $r \ll n$:
\begin{itemize}
    \item \textbf{QR Decomposition:} A standard choice that performs a $QR$ factorization of $X_t + D_t$ and extracts the orthogonal factor $Q$ \citep[(4.8)]{AMS:2008}.
    \item \textbf{Polar Decomposition:} Maps the tangent vector to the manifold by finding the closest orthogonal matrix in the Frobenius norm, typically implemented via iterative Newton--Schulz methods \citep{ZhuSato2020}.
    \item \textbf{Cayley Transform:} An algebraic alternative using a skew-symmetric mapping. When implemented with the Sherman--Morrison--Woodbury (SMW) identity, it avoids large matrix inversions, reducing the cost to a $2r \times 2r$ system \citep[\S 2.2]{WenYin2013}. It can also be computed implicitly via the fixed-point iteration $ Y_{k+1} = X + \frac{\alpha}{2} \Omega \bigl( X + Y_k \bigr) $, which converges quadratically as $o(\alpha^{2+k})$.
\end{itemize}

In \Cref{sec:retractions} and in the right part of \Cref{tab:gpt_pretrain}, we present numerical results comparing different kinds of retractions. In the remaining numerical experiments, we use the Cayley retraction as the standard choice, approximated via fixed-point iteration, because of its simplicity of implementation and good performance in the comparison tests. Moreover, our analysis shows that the framework is robust to \textit{approximate} retractions: as established in \Cref{thm:convex_regretbound}, the introduction of a maximal Frobenius error $\delta$ in the retraction mapping merely adds a manageable linear term to the regret bound. This theoretical guarantee justifies using truncated or iterative retraction methods that can run for only a few iterations without reaching machine precision, while offering significant speedups.

\section{Theoretical Guarantees}

\subsection{Regret Analysis}
In this section, we present a convex regret analysis for \Cref{algo:StiefelAdamW_BA}. Regret analysis is a standard tool in convex optimization that quantifies how much an optimization algorithm, when running dynamically on a family of convex objective functions $\cL_t$, is suboptimal with respect to the optimal objective ahead of time. In particular, given a sequence of iterates $\{W_t\}_{t=1,\dots,T}$, we define the regret function as
\[
R(T) \coloneqq \sum_{t=1}^T \cL_t(W_t) - \min_{W} \sum_{t=1}^T \cL_t(W).
\]
We recall that an algorithm is said to be zero regret if $R(T)/T \to 0$ as $T \to +\infty$. Our aim is to show that \Cref{algo:StiefelAdamW_BA} indeed produces arbitrarily small regret for a small enough learning rate and retraction error.
To prove this result, we will make the following assumptions:
\begin{assumptions}[Setting of regret analysis]\label[assumptions]{assumptions:regret}
    \hfill
    \begin{enumerate}
    \item[(H1)] The family $\cL_t \colon \R^{m \times r} \times \Stnr \to \R$ is the restriction of a family of Euclidean strictly convex functions defined on $\R^{r \times m} \times \R^{r \times n}$. By a small abuse of notation, we will also denote the extension family with $\cL_t$, and we will denote the minimizer with $(B^*, A^*)$.
    \item[(H2)] The first momentum coefficients $\beta_{1,t} = \beta_1 b^t$ decrease geometrically in time for a constant $0<b<1$, and with $\beta_1< \sqrt{\beta_2}$.
    \item[(H3)] The iterates $B_t$ stay bounded, i.e., $\sup_{t}\|B_{t} \|_{\max} \leq D_{\infty}$.
    \item[(H4)] The Euclidean gradient $\nabla \cL_t(B_t,A_t)$ stays bounded, i.e., $ \sup_{t}\|\nabla \cL_t(B_{t},A_t) \|_{\max} \leq G_{\infty} $.
    \item[(H5)]\label{hp:adamax_hp} The second momentum update in \Cref{algo:StiefelAdamW_BA} is followed by an entrywise maximum, i.e., $V_{t+1} = \max(\beta_2 V_{t-1} + (1-\beta_2)\, G_t^{2}, \ V_{t-1}),$
    as in AMSGrad \cite[Algorithm 2]{convergence_adam_beyond}.
    \item[(H6)] The projected direction is aligned with the globally correct direction, $ \langle \xi_t, A^*-A_t \rangle \geq 0 $.
\end{enumerate}
\end{assumptions}
We emphasize that \Cref{assumptions:regret} are fairly standard assumptions used to study convergence of Adam-like algorithms, and they were already employed in, e.g., \cite{convergence_adam_beyond}. Assumption (H6) is a hypothesis often used in Euclidean optimizers to ensure that the current local descent direction is aligned with the global direction to the minimizer.

\begin{theorem}(Regret bound)\label{thm:convex_regretbound}
    Under~\Cref{assumptions:regret}, consider the sequence of iterates produced by \Cref{algo:StiefelAdamW_BA} with decreasing learning rates $\eta_t = \eta/\sqrt{t}$, and no weight decay. Then,
    \begin{align}\label{eq:bound_regret_thm}
        R(T) &\leq   C_1 + C_2\sqrt{T} +C_3 \sqrt{1+\log T} + 
        C_4 \log T + C_5T^{-1/2},
    \end{align}
    where $C_1,C_2,C_3,C_4,C_5$ are constants independent of $T$. In particular,   $\lim_{T \to + \infty}R(T)/T =0$.
\end{theorem}

We note that, although the theoretical result requires an AMSGrad-like assumption (H5), in practice the algorithm can be used without the $\max$ update with no loss in performance. The proof of \Cref{thm:convex_regretbound} can be found in~\Cref{app:regret_proof}.

In most practical implementations, the retraction is computed only approximately via a numerical algorithm, such as the fixed-point method used in most of our experiments; see also~\Cref{app:cayley_approx_error}. The proof of \Cref{thm:convex_regretbound} above extends straightforwardly to that case: assuming that the computed retraction has an error of $\delta>0$, then a term $C_6\delta$ has to be added to \eqref{eq:bound_regret_thm}, without significantly affecting the main result of the theorem.

\subsection{Gradient Boundedness and Stability to Large Learning Rates}\label{sec:boundedness}
Working with an orthonormal factor yields a method with bounded gradients, potentially improving stability at large learning rates. Here, we make this point more concrete with an example. Let us consider the rank-$r$ recovery problem in $\R^{n \times n}$,
\[
    \cL(A,B) = \tfrac{1}{2}\|BA - \alpha I_n\|_{\mathrm{F}}^2,
    \qquad B \in \R^{n \times r},\ A \in \R^{r \times n},
\]
trained by plain gradient descent from the canonical LoRA initialization $B_0 = 0$, $A_0$ arbitrary. The first GD step gives $B_1 = \eta\alpha A_0\tr$ and $A_1 = A_0$, so for the scaling $\alpha = 1/\eta$ the two factors align already after one step. Specifically, we have $B_1 = A_1\tr$ and a direct induction argument shows that the iterates preserve $B_k = A_k\tr$ thereafter. Setting $X_k \coloneqq B_k = A_k\tr$, the dynamics take the form $X_{k+1} = X_k - \eta (X_k X_k^\top-\eta^{-1}I)X_k = (2I-\eta X_kX_k^\top)X_k$, thus,  using the SVD $X_k = U_k \Sigma_k V_k\tr$, the dynamics decouple across singular values into the scalar recursion $ \sigma_{i,k+1} = \sigma_{i,k}\bigl(2 - \eta\,\sigma_{i,k}^{2}\bigr) $, which diverges as soon $|2 - \eta\,\sigma_{i,k}^{2}|>1$, i.e., when some $\sigma_{i,k} > \sqrt{3/\eta}$. The largest stable learning rate is therefore dictated by the largest singular value of the iterate, a property of the \emph{parameterization}, not of the underlying optimization landscape.

If instead the constraint $A_k A_k\tr = I_r$ is enforced, the gradient descent update on $B$ simplifies to the affine recursion $ B_{k+1} \;=\; (1-\eta)\, B_k + \eta\alpha\, A_k\tr $, which is bounded for every $0 < \eta < 2$ as $\|A_k\|=1$, regardless of singular-value scale, removing the dependence of the stable learning rate on the iterate. 
This phenomenon is well documented in the Riemannian optimization literature \cite{schotthoefer2025}, and we provide a more precise result in the next~\Cref{prop:bounded_gradient}, whose proof can be found in~\Cref{app:proof_bounded_gradient}.

\begin{proposition}[Gradient boundedness on fibers]
\label[proposition]{prop:bounded_gradient}
Consider the maps $ \Phi \colon \R^{m \times r} \times \R^{r \times n}\to \Rmn $ given by $\Phi(B,A) = BA $ and the map $ \widetilde \Phi = \Phi|_{\Stnr \times \R^{m \times r}} $. Let $W \in \Rmn$ be a fixed matrix with $\rank(W) \leq r$ such that $\nabla \cL(W) \ne 0$. Let $\mathcal F\coloneqq\Phi^{-1}(W)$ and $\widetilde{\mathcal F}\coloneqq \widetilde \Phi^{-1}(W)$. Then,
\[
    \|\nabla (\cL \circ\Phi)\|_{L^\infty(\mathcal F)} = +\infty, \quad \|\nabla (\cL \circ\widetilde \Phi)\|_{L^\infty(\widetilde{\mathcal F})} <+ \infty.
\]
\end{proposition}
In particular, it is known that the boundedness of $\nabla (\cL \circ \Phi)$ is closely related to the range of stable learning rates. This suggests that optimization algorithms on the parameterization $\widetilde \Phi$ are more stable than the ones on the more redundant representation $\Phi$.

\section{Numerical Experiments}\label{sec:experiments}
In this section, to show the effectiveness and scalability of \algname, we present several numerical experiments for both fine-tuning pretrained models with LoRA adapters \citep{hu2022lora} and LLM pretraining. We compare against six baselines: standard \textbf{AdamW} \citep{Loshchilov:2019}; \textbf{Scaled AdamW} \citep{ZhangPilanci2024}, which introduces a coupled preconditioner accounting for the product structure; \textbf{GeoLoRA} \citep{schotthofer2025geolora}, which uses two Stiefel representations; \textbf{LoRA-RITE} \citep{yen2025lora} and \textbf{LoRA-Pro} \citep{Wang2025lorapro}, which modify the gradient structure to reduce sensitivity to the non-uniqueness of the factorization; and \textbf{Cayley Adam} \citep{LFT:2020}, which uses the Riemannian Adam variant proposed in \citep{BecigneulGanea:2019} on the Stiefel manifold. We emphasize that the latter method uses a scalar preconditioner, whereas \algname{} uses a diagonal preconditioner.

\subsection{LoRA Fine-Tuning}\paragraph{GPT2} In this experiment, we tested \algname{} for fine-tuning GPT2 on the E2E Natural Language Generation challenge \citep{novikova2017e2e} with LoRA of rank 4. We report the results in \Cref{tab:table_1}. In all experiments, we trained the models for $5$ epochs with a batch size of $8$. For details on the hyperparameter settings, see \Cref{tab:GPT_finetuning_hyperparams}.
As shown in \Cref{tab:e2e}, \algname{} outperforms all baselines across all tasks except ROUGE-L. Interestingly, we observe that, consistently with the findings of \cite{ZhangPilanci2024} for Scaled AdamW, \algname{} achieves better performance when using more aggressive moving average parameters $\beta_1$, $\beta_2$. This suggests that part of AdamW's update may be spent along invariant directions, while Riemannian-informed approaches avoid this and allow greater emphasis to be placed on the current gradient direction. All numerical experiments were performed on a single NVIDIA A100 80GB, except for GPT2 pretraining, which was performed on two NVIDIA H100 80GB via Modal.

\begin{table*}[ht]
    \centering
    \caption{Fine-tuning performance with low-rank adapters. Best results highlighted in bold. \textbf{Left:} GPT2 on the E2E Natural Language Generation challenge, with rank = 4. AdamW and Scaled AdamW are reported from \cite{ZhangPilanci2024}. \textbf{Right:}  ViT-Base on CIFAR-10, for three different choices of adapter~rank. We report with $\pm \sigma$ the standard deviation over $5$ random initializations.}
    \resizebox{\textwidth}{!}{
    \begin{tabular}{l c c c c c}
    \toprule
    \multicolumn{6}{c}{ $\,$} \\
    \toprule
    \textbf{Method} & \textbf{BLEU} & \textbf{NIST} & \textbf{MET} & \textbf{ROUGE-L} & \textbf{CIDEr} \\
    \midrule
    AdamW \cite{Loshchilov:2019} & 68.41 {$\pm$ 0.4950} & 8.65 {$\pm$ 0.04} & 46.38 {$\pm$ 0.12} & 71.13 {$\pm$ 0.17} & 2.51 {$\pm$ 0.001}\\
    Scaled AdamW \cite{ZhangPilanci2024}     & \textbf{69.17} {$\pm$ 0.43} & 8.72 {$\pm$ 0.058} & 46.44 {$\pm$ 0.16} & \textbf{71.57} {$\pm$ 0.235} & 2.51 {$\pm$ 0.005} \\
    \textbf{\algname} & 69.20 {$\pm$ 0.964} & \textbf{8.74} {$\pm$ 0.102} & \textbf{46.48} {$\pm$ 0.259} & 71.40 {$\pm$ 0.44} & \textbf{2.51} {$\pm$ 0.02} \\
    GeoLoRA \cite{schotthoefer2025} & 68.11 {$\pm$ 0.271} & 8.60 {$\pm$ 0.07} & 45.82 {$\pm$ 0.24} & 70.47 {$\pm$ 0.388} & 2.41 {$\pm$ 0.02} \\
    LoRA-RITE \cite{yen2025lora} & 68.98 {$\pm$ 1.03} & 8.69 {$\pm$ 0.11} & 46.35 {$\pm$ 0.219} & 71.17 {$\pm$ 0.410} & 2.48 {$\pm$ 0.05} \\
    LoRA-Pro \cite{Wang2025lorapro} & 68.12 {$\pm$ 0.25} & 8.61 {$\pm$ 0.06} & 45.82 {$\pm$ 0.250} & 70.46 {$\pm$ 0.382} & 2.42 {$\pm$ 0.02} \\
    Cayley Adam \cite{LFT:2020} & 68.97 {$\pm$ 1.04} & 8.69 {$\pm$ 0.11} & 46.36 {$\pm$ 0.215} & 71.16 {$\pm$ 0.415} & 2.48 {$\pm$ 0.05} \\
    \bottomrule
    \end{tabular}
    \hspace{3em}
    \begin{tabular}{ccc}
    \toprule
    \multicolumn{3}{c}{\textbf{Rank}} \\
    \toprule
    \textbf{32} & \textbf{64} & \textbf{128} \\
    \midrule
     95.6 {$\pm 0.2$}          & 95.55 {$\pm {0.15}$}          & 95.82 {$\pm {0.29}$}          \\
     92.25 {$\pm 0.28$}          & 94.83 {$\pm {0.31}$}          & 95.30 {$\pm {0.12}$}          \\
     \textbf{95.91} {$\pm0.19$} & \textbf{96.04} {$\pm {0.15}$} & \textbf{96.41} {$\pm {0.10}$} \\
     95.53 {$\pm 0.44$}          & 95.12 {$\pm {0.37}$}          & 95.27 {$\pm {0.23}$}          \\
     94.55 {$\pm 0.32$}          & 94.98 {$\pm {0.10}$}          & 94.89 {$\pm {0.24}$}          \\
     90.17 {$\pm 2.30$}          & 94.37 {$\pm {0.26}$}          & 94.32 {$\pm {0.21}$}          \\
     95.44 {$\pm 0.92$ }          & 95.71 {$\pm {0.17}$}          & 96.12 {$\pm {0.11}$}          \\
    \bottomrule
    \end{tabular}
    }
    \label{tab:table_1}
\end{table*}

\paragraph{Vision Transformers}\label{subsec:vit_ft}

In this experiment, we fine-tuned the base Vision Transformer from \cite{dosovitskiy2021an} on CIFAR-10 \cite{krizhevsky2009learning}, with results shown in the right part of \Cref{tab:table_1}. In \Cref{fig:vit}, we compare loss descent and time per iteration against the best loss achieved. All models have been trained for $50$ epochs with a batch size of $64$, LoRA alpha $32$, learning rate $10^{-3}$, and no scheduler. For all optimizers, we used weight decay of $10^{-5}$ on all adapters, applied to the key-query attention matrices, attention projection, and the last two fully connected layers. We did not optimize biases and left them as in the pretrained model.
As we can observe from the results in the right panel of \Cref{tab:e2e} and \Cref{fig:vit}, \algname{} is able to outperform all baselines in terms of performance, with convergence speed comparable to that of Scaled AdamW \citep{ZhangPilanci2024} and AdamW \citep{Loshchilov:2019}. 

\begin{figure*}[t]
    \centering
    \caption{\textbf{Left and center:} Loss function descent for ViT Base on CIFAR-10 LoRA fine-tuning (LoRA rank $64$ and learning rate $10^{-3}$). \textbf{Right:} Learning rate versus best training loss for different optimizers when fine-tuning ViT Base on CIFAR-10. In this experiment, we fixed the LoRA rank at $32$ and trained all models for $50$ epochs.
    }
    \begin{tabular}{ccc}
        \includegraphics[height=0.28\columnwidth]{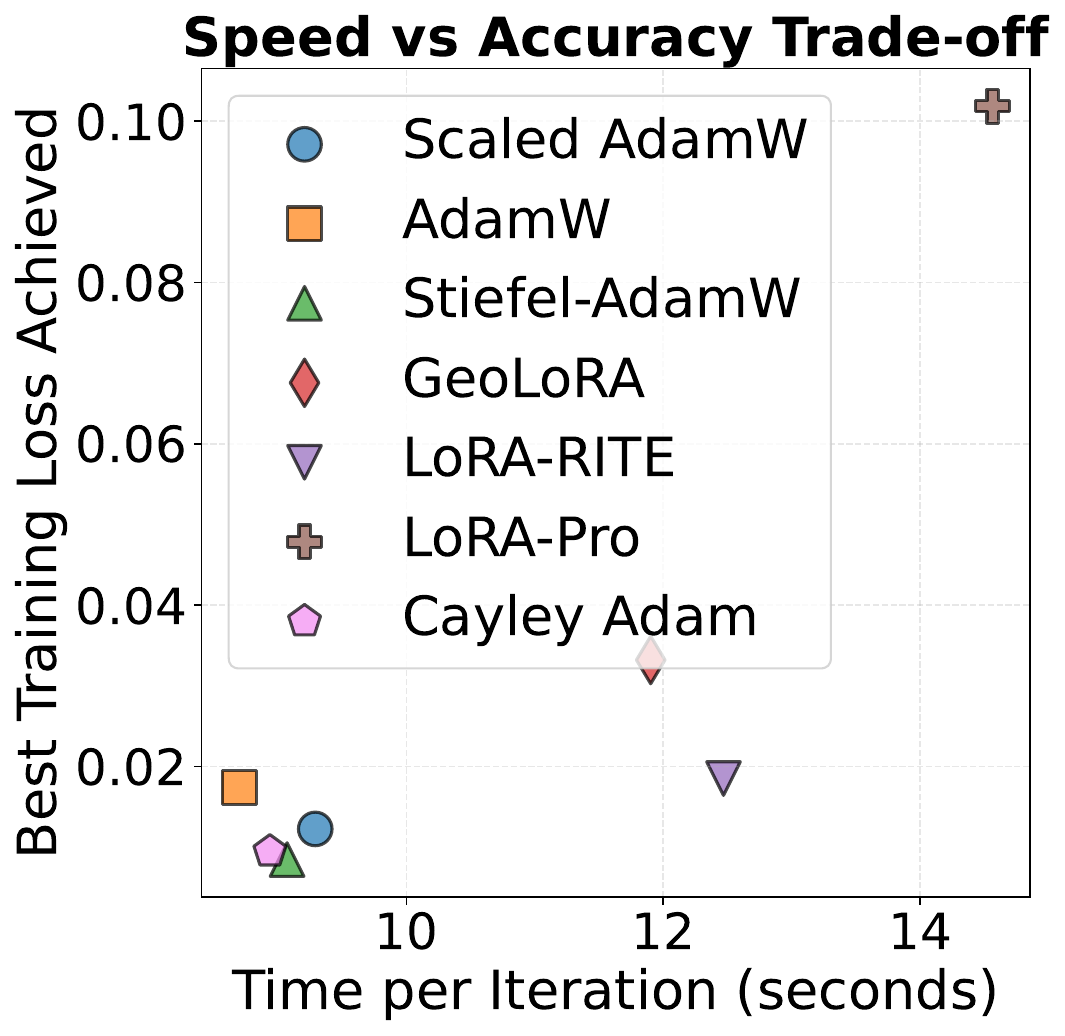} &\includegraphics[height=0.28\columnwidth]{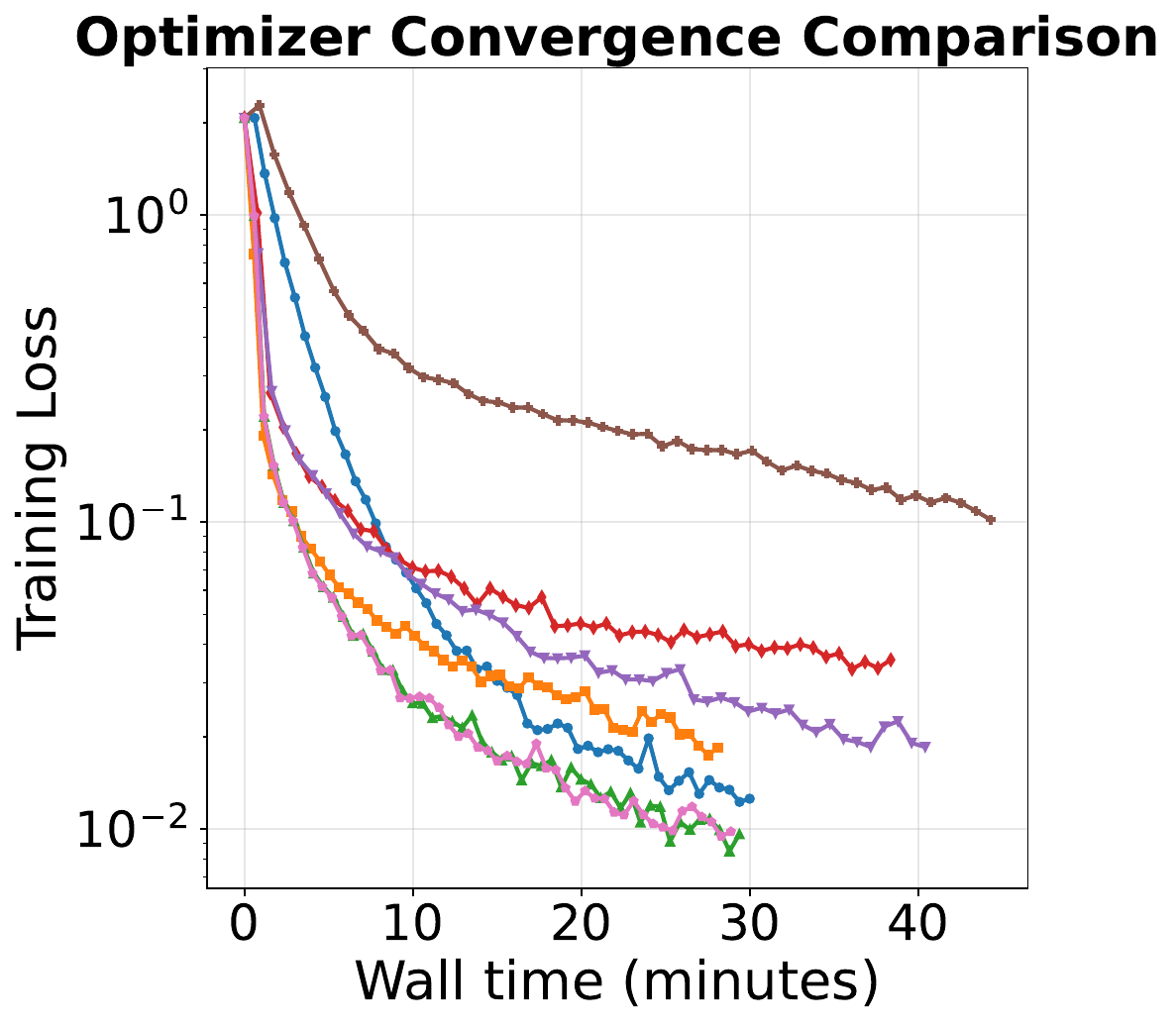}  &\includegraphics[height=0.28\columnwidth]{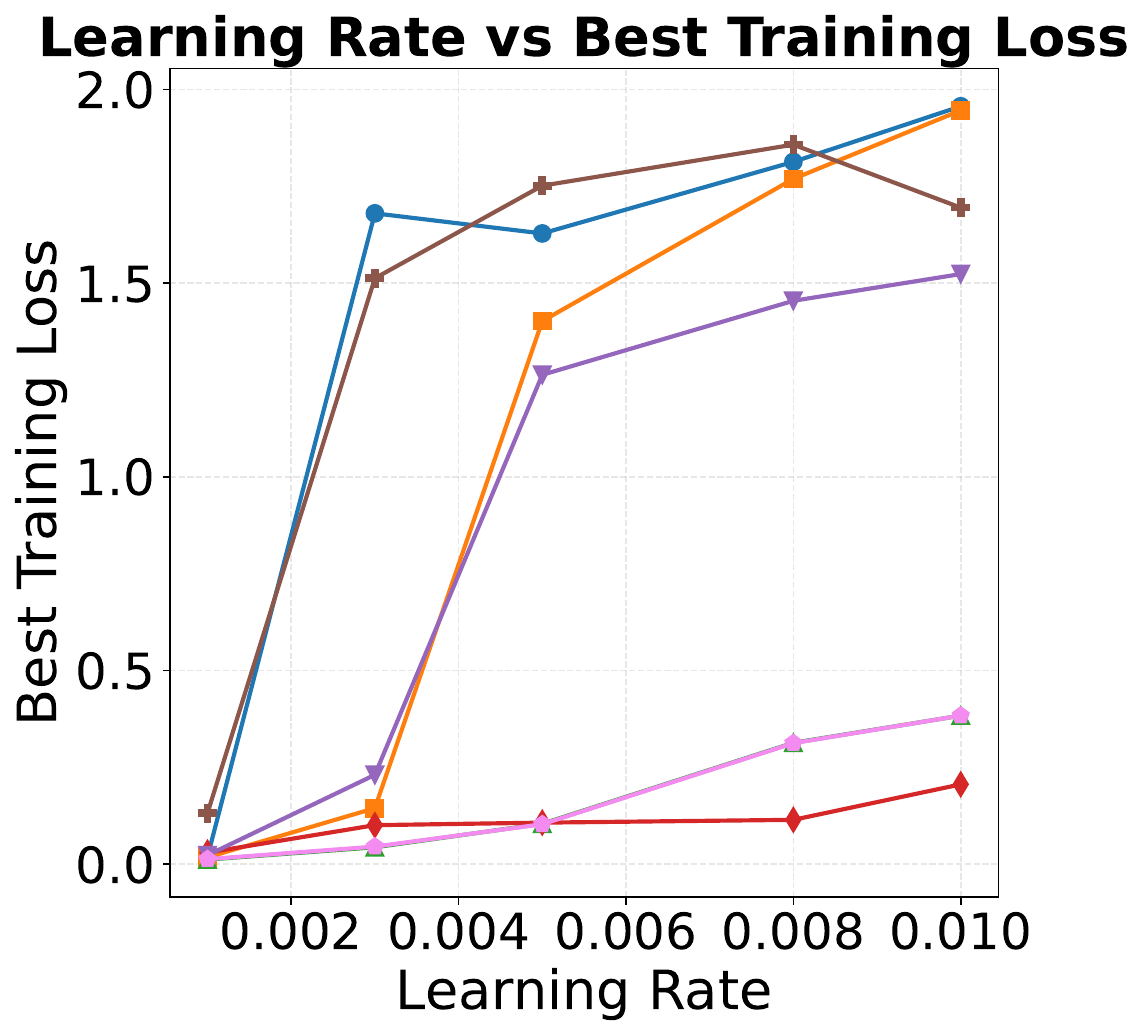}
    \end{tabular}
    \label{fig:vit}
\end{figure*}


\paragraph{Mistral 7B}

In this experiment, we tested the effectiveness of \algname{} for fine-tuning Mistral 7B \citep{jiang2023mistral7b} on the GLUE benchmark \citep{wang2019gluemultitaskbenchmarkanalysis} for natural language understanding, following the implementation in \citep{ZhangPilanci2024}. LoRA adapters of rank $16$ have been applied to all query, key, value projection, and gate matrices of multihead attention. We did not train biases, and for all optimizers, we set the LoRA alpha learning-rate scaling parameter to $16$. We used mixed precision for all optimizers: the base model was loaded in its $4$-bit quantized version, optimizer states were in float$32$, and operations were performed in mixed-precision bfloat16. We trained all models using the codebase of \cite{ZhangPilanci2024}, with a dropout of $0.1$ and a batch size of $8$ across all models. For AdamW, Scaled AdamW, and GeoLoRA, we used the optimal hyperparameters (learning rate and $\beta_1,\beta_2$) from~\cite{ZhangPilanci2024}; for \algname{}, we used $\beta_1 = \beta_2 = 0.95$. For each GLUE task, we report the value of the standard test metric (either accuracy or correlation) and the percentage deviation from the best performer among all optimizers.
As shown in \Cref{tab:mistral_results}, \algname{} outperforms all baselines in terms of average score. On single tasks, \algname{} outperforms all baselines on MNLI, MRPC, STS-B, and WNLI, while maintaining a competitive performance on all other tasks.

\begin{table*}[t]
    \centering
    \caption{Scores for rank $16$ LoRA fine-tuning of the 4-bit quantized Mistral 7B model on the GLUE benchmark for Natural Language Understanding (NLU) challenges with different optimizers (best results highlighted in bold). In parentheses, we report, for each task, the percentage deviation from the best performer. SGD, Scaled GD, and Scaled AdamW results reported from \cite[Table 2]{ZhangPilanci2024}.}
    \resizebox{\textwidth}{!}{
    \begin{tabular}{lcccccccccc}
        \toprule
        \multirow{2}{*}{\textbf{Method}} & 
        \multicolumn{10}{c}{\textbf{GLUE}}  \\ 
        \cmidrule{2-11}
        &    \textcolor{black}{MNLI} & \textcolor{black}{SST-2} & \textcolor{black}{MRPC} & \textcolor{black}{CoLA} & \textcolor{black}{QNLI} & \textcolor{black}{QQP} & \textcolor{black}{RTE} & \textcolor{black}{STS-B} & \textcolor{black}{WNLI} & Avg.\\
        \midrule
        
        \multirow{2}{*}{SGD}  
        & 88.15 & 96.10 & 70.10 & 55.89 & 94.22 & 88.59 & 50.90 & 47.64 & 49.30 & \multirow{2}{*}{71.21} \\
        & (-4.14\%) & (-1.18\%) & (-22.07\%) & (-22.23\%) & (-1.21\%) & (-3.94\%) & (-44.27\%) & (-48.37\%) & (-43.54\%) & \\[0.3em]
        
        \multirow{2}{*}{Scaled GD \cite{ZhangPilanci2024}}  
        & 90.21 & 96.90 & 81.62 & 68.17 & 94.40 & 91.15 & 54.15 & 90.31 & 56.34 & \multirow{2}{*}{80.36} \\
        & (-1.89\%) & (-0.36\%) & (-9.26\%) & (-5.14\%) & (-1.01\%) & (-1.16\%) & (-40.72\%) & (-2.13\%) & (-35.48\%) & \\[0.3em]
        
        \multirow{2}{*}{AdamW}  
        & 91.64 & \textbf{97.25} & 87.01 & \textbf{71.87} & 94.79 & 91.81 & 90.25 & 90.51 & 85.91 & \multirow{2}{*}{89.00} \\
        & (-0.34\%) &  & (-3.27\%) &  & (-0.61\%) & (-0.44\%) & (-1.19\%) & (-1.92\%) & (-1.61\%) & \\[0.3em]
        
        \multirow{2}{*}{Scaled AdamW \cite{ZhangPilanci2024}}  
        & 90.68 & \multirow{2}{*}{\textbf{97.25}} & 89.46 & 71.30 & 94.67 & \multirow{2}{*}{\textbf{92.22}} & \multirow{2}{*}{\textbf{91.34}} & 91.10 & 83.10 & \multirow{2}{*}{89.01} \\
        & (-1.38\%) &  & (-0.55\%) & (-0.79\%) & (-0.73\%) &  &  & (-1.28\%) & (-4.83\%) & \\[0.3em]
        
        \multirow{2}{*}{\textbf{\algname}}  
        & \multirow{2}{*}{\textbf{91.95}} & 96.79 & \multirow{2}{*}{\textbf{89.95}} & 70.61 & 94.78 & 91.83 & 90.61 & \multirow{2}{*}{\textbf{92.28}} & \multirow{2}{*}{\textbf{87.32}} & \multirow{2}{*}{\textbf{89.57}} \\
        &  & (-0.47\%) &  & (-1.75\%) & (-0.62\%) & (-0.42\%) & (-0.80\%) &  &  & \\[0.3em]
        
        \multirow{2}{*}{GeoLoRA \cite{schotthofer2025geolora}}  
        & 91.30 & 94.61 & 87.26 & 69.78 & \multirow{2}{*}{\textbf{95.37}} & 90.80 & 88.81 & 91.45 & \multirow{2}{*}{\textbf{87.32}} & \multirow{2}{*}{88.52} \\
        & (-0.71\%) & (-2.71\%) & (-2.99\%) & (-2.91\%) &  & (-1.54\%) & (-2.77\%) & (-0.90\%) &  & \\[0.3em]
        
        \bottomrule
    \end{tabular}
    }
    \label{tab:mistral_results}
\end{table*}
%
%
\subsection{GPT2 Pretraining}
As the proposed method works for all problems in which the manifold of $\rank$-$r$ matrices appears, it also applies directly to pretraining transformer-based architectures. In particular, self-attention naturally respects this structure, as the image of the map $(W_Q,W_K) \in \R^{n \times r} \times \R^{n \times r} \mapsto W_Q W_K\tr$ is exactly the set $\cMr$. Despite this, the parametrization map is highly non-injective, and therefore the problem could be restated equivalently by minimizing on $\cMr \cong \mathrm{St}(n,r) \times \R^{n \times r}_*/\mathrm{O}(r)$ instead of $\R^{n \times r} \times \R^{n \times r}$. In contrast, we optimize all other parameters, such as biases or non-structured matrices, with the standard AdamW step.
In \Cref{tab:gpt_pretrain} we present the results for pretraining GPT2 \citep{radford2019language} on OpenWebText \citep{Gokaslan2019OpenWeb} using Karpathy's reproduction\footnote{https://github.com/karpathy/nanoGPT}. We reproduced the pretraining for \algname{} using exactly the same AdamW hyperparameters from the repository. We trained both AdamW and \algname{} for $7000$ iterations. As shown, \algname{} produces results comparable to standard AdamW \citep{Loshchilov:2019} in both performance and peak GPU memory usage. We performed no hyperparameter tuning, and \algname{} uses the same hyperparameters as AdamW, as shown in the reproduced repository. This experiment highlights that \algname{} is not limited to fine-tuning scenarios such as LoRA, but can be applied directly in pretraining settings where low-rank structure arises naturally.

\begin{table*}[t]
    \centering
    \caption{\textbf{Left}: Full GPT2 pretraining on OpenWeb Text for 6000 iterations. \textbf{Right:} Ablation over different retractions on ViT Base.}
    \resizebox{\textwidth}{!}{
        \begin{tabular}{l c c}
            \toprule
            \multirow{2}{*}{\textbf{Method}}& \multirow{2}{*}{\textbf{Test Loss $\pm \sigma$}} & \textbf{Peak Memory} \\
            & & (GB)\\
            \midrule
            AdamW            & 3.260 $\pm$ 0.05 & 13.8\\
            \textbf{\algname} & \textbf{3.236} $\pm$ 0.07 & 13.8 \\
            \bottomrule
        \end{tabular} 
    \label{tab:gpt_pretrain}
    \hspace{4em}
        \begin{tabular}{l c c c c}
            \toprule
            \multirow{2}{*}{\textbf{Retraction} }  & \textbf{Test acc.} & \textbf{Test acc. }    & \textbf{Test acc. }  \\ 
            & ($r = 32$) & ($r = 64$) & ($r = 128$)\\
            \midrule
            Cayley FP        & 96.19               & 96.25 & 96.45 \\ 
            Cayley SMW       & 96.11               & 96.09 & 96.52 \\ 
            Cayley direct    & 96.11               & 96.22 & 96.61 \\ 
            QR               & 95.41               & 94.83 & 93.71 \\ 
            Polar            & 96.11               & 96.22 & 96.61 \\ 
            Newton--Schulz   & 96.11               & 96.20 & 96.52 \\ 
            \bottomrule
        \end{tabular}
    }
    \label{tab:e2e}
\end{table*}

\subsection{Qwen2 Pretraining}

To showcase compatibility of our proposed approach with the presence of positional embeddings such as RoPE \cite{rope}, in \Cref{tab:qwen} we also present numerical results for Qwen2~\cite{yang2024qwen2technicalreport} LoRA pretraining on WikiText-103 \cite{merity2016pointer}.
The model with the adapters has 100M parameters; it was trained for 20K steps with a batch size of $16$ and two steps of gradient accumulation.
As shown in \Cref{tab:qwen}, \algname{} consistently outperforms the Euclidean version of AdamW in both mean and variance.

\begin{table}[!ht]
    \centering
    \caption{Qwen2 pretraining on Wikitext-103, standard deviation reported over $5$ random seeds.}
    \begin{tabular}{l c}
            \toprule
            \multirow{2}{*}{\textbf{Method}}& \multirow{2}{*}{\textbf{Loss $\pm \sigma$}} \\ \\
            \midrule
            AdamW            & 2.66 $\pm$ 0.05 \\
            \textbf{\algname} & \textbf{2.59} $\pm$ 0.004  \\
            \bottomrule
        \end{tabular} 
    \label{tab:qwen}
\end{table}

\vspace{-1em}

\subsection{Stepsize Stability}
Motivated by \Cref{prop:bounded_gradient}, we show the stability of \algname{} with respect to the learning rate, numerically demonstrating that methods with compact fibers are more stable with respect to learning-rate size. \\
We fine-tuned multiple vision transformers (ViT Base) on CIFAR-10 for different learning rates, keeping all other hyperparameters fixed as in \Cref{subsec:vit_ft}. In the right panel of \Cref{fig:vit}, we plot the learning rate against the best loss obtained during training. As expected, the results in \Cref{fig:vit} show that ``pure'' Riemannian methods such as GeoLoRA \cite{schotthofer2025geolora} and the proposed \algname{} are more stable with respect to learning-rate magnitude. \\
In particular, Riemannian methods in which the fiber of $\Phi$ is not compact (such as AdamW and Scaled AdamW, which are defined in $\R^{r \times n} \times \R^{m \times r}$), appear to be less stable with respect to larger learning rates, despite the preconditioning (i.e., different metric) employed in Scaled AdamW \cite{ZhangPilanci2024}.

\section{Conclusions, Limitations, and Future Work}

In this work, we presented \algname, a stochastic Riemannian variant of AdamW that naturally provides convergence guarantees. The compactness of the space in which one factor lives helps avoid potential numerical instabilities that can arise from unbalanced initializations. One key advantage of the proposed method is its simplicity of implementation, which requires only minimal machinery from Riemannian optimization theory while maintaining guarantees.\\
We demonstrated the method's effectiveness and scalability across a range of problems, from fine-tuning to pretraining LLMs. One limitation of the current method is that it still retains a set of orthogonal invariances, which stems from the entrywise nature of the Adam algorithm. While this can be solved by imposing a gauge condition and working on horizontal spaces on the quotient space $\cMr$, it requires additional computational effort, as the gradient needs to be computed using joint information on the pair $(B,A)$, $ (\nabla_B \cL, \nabla_A \cL) $, and therefore does not allow for full parallelization.
Future research could extend the current method to different preconditioners and propose a similarly simple version that is fully invariant on the manifold of fixed-rank matrices and can partially mitigate this extra computational effort.

\begingroup
    \small
    \bibliographystyle{alpha}
    \bibliography{arXiv_biblio}
\endgroup

\newpage
\appendix
\onecolumn
\section{Appendices}

\subsection{Geometry of the Stiefel Manifold} \label[appendix]{app:stiefel_geometry}
The (column-orthonormal) Stiefel manifold is defined as
\[
    \Stnr \coloneqq \{ X \in \R^{n \times r} \colon X\tr \! X = I_{r} \}.
\]
It is a smooth embedded submanifold of $\R^{n \times r}$ of dimension $ nr - \tfrac{1}{2} r(r+1) $.

The tangent space at a point $ X \in \Stnr $ is given by
\[
    \T_{X} \Stnr = \left\{ \xi \in \R^{n \times r}  \colon  X\tr \! \xi + \xi\tr \! X = 0
    \right\}.
\]
Equivalently, any tangent vector $\xi \in \T_{X} \Stnr$ can be decomposed as
\[
    \xi = X \Omega + X_{\perp} K,
\]
where $ \Omega \in \R^{r \times r} $ is skew-symmetric, $X_\perp \in \R^{n \times (n-r)}$ satisfies $ [X \; X_\perp] \in \mathrm{O}(n) $, $ \mathrm{O}(n) $ being the orthogonal group, and $ K \in \R^{(n-r) \times r} $.

\subsubsection{Orthogonal Projection onto \texorpdfstring{$ \T_{X}\Stnr $}{TX St(n,r)}}
For any matrix $ \xi \in \R^{n \times r} $, its orthogonal projection onto $ \T_{X} \Stnr $ with respect to the Euclidean inner product is given by
\begin{equation}\label{eq:proj_stiefel}
    \Proj_{X}(\xi) = X \, \mathrm{skew}(X\tr \! \xi) + (I_{n} - XX\tr)\,\xi,
\end{equation}
where $ \mathrm{skew}(M) \coloneqq \tfrac{1}{2}(M - M\tr) $ denotes the skew-symmetric part of a square matrix.

Equation~\eqref{eq:proj_stiefel} admits the equivalent and more compact expression
\[
    \Proj_{X}(\xi) = \xi - X \, \mathrm{sym}(X\tr\! \xi),
\]
where $\mathrm{sym}(M) \coloneqq \tfrac{1}{2}(M + M\tr)$ denotes the symmetric part. This formula is commonly used in Riemannian optimization on the Stiefel manifold; see, e.g., \citep[Prop.~3.6.1]{AMS:2008}.

\subsubsection{Row-Orthonormal Stiefel Manifold}
In this work, we enforce a \emph{row-orthonormal} constraint on the factor $A \in \R^{r \times n}$, namely,
\[
    A A\tr = I_r.
\]
This corresponds to working with the transpose variable $ X \coloneqq A\tr \in \Stnr$. All Riemannian operations (projection, retraction, and gradient computation) are therefore performed on $X$, and the updated factor is recovered as $ A = X\tr $.

\subsection{Cayley Transform and Retraction} \label[appendix]{app:cayley}
A commonly used retraction on the Stiefel manifold $ \Stnr = \{ X \in \R^{n \times r} \colon X\tr \! X = I_{r} \}$ is based on the Cayley transform.
The Cayley transform generates a smooth curve on the Stiefel manifold by exponentiating a skew-symmetric matrix in a rational form, thereby avoiding explicit matrix exponentials.

The closed-form Cayley retraction at a point $ X \in \Stnr $ is defined as
\begin{equation}\label{eq:exact_cayley_transform}
   Y(\alpha) = \left( I_{n} - \tfrac{\alpha}{2} \Omega \right)^{-1}
     \left( I_{n} + \tfrac{\alpha}{2} \Omega \right) X,
\end{equation}
where $ \Omega \in \R^{n \times n} $ is a skew-symmetric matrix and $\alpha \ge 0$ is a step-size parameter.
The curve satisfies
\[
   Y(0) = X,  \qquad \frac{\mathrm{d}}{\mathrm{d}\alpha} Y(\alpha)\big|_{\alpha=0} = \Omega X,
\]
and therefore defines a valid first-order retraction on the Stiefel manifold~\citep{AMS:2008}.

Computing the closed-form expression~\eqref{eq:exact_cayley_transform} requires solving a linear system involving an $n \times n$ matrix, which can be computationally expensive for large $n$.
A fixed-point approximation of the Cayley transform is given by
\begin{equation} \label{eq:cayley_fixed_point_map}
   Y(\alpha) = X + \frac{\alpha}{2} \, \Omega \bigl( X + Y(\alpha) \bigr).
\end{equation}
Starting from the initialization $ Y_{0} = (I + \alpha \Omega) X $, this fixed-point equation can be solved with a small number of iterations, each involving only matrix multiplications.
In practice, only a few iterations are sufficient to obtain an accurate approximation of the exact Cayley retraction.

\subsubsection{Cayley Retraction in Our Setting}
As mentioned in the main text, in the factorization of a weight matrix, we work with the row-orthonormal factor of size $ r $-by-$ n $ satisfying $ A A\tr = I_{r} $. For convenience, we switch to a column-orthonormal representation so that $ X \coloneqq A\tr \in \Stnr $.

Let $ \xi \in \T_{X_{t}} \Stnr $ be a tangent vector at $ X_{t} \coloneqq A_{t}\tr $.
Following \citep[Eq.~(2)]{LFT:2020}, the Cayley retraction used in our algorithm is defined as
\[
    X_{t+1} = \left( I_n - \tfrac{1}{2} \Omega \right)^{-1} \left( I_n + \tfrac{1}{2} \Omega \right) X_{t}, \qquad  A_{t+1} \coloneqq X_{t+1}\tr,
\]
where the skew-symmetric matrix $\Omega$ is constructed as
\[
    \widehat{\Omega} = \xi X_{t}\tr - \tfrac{1}{2} X_{t} ( X_{t}\tr \xi X_{t}\tr),
    \qquad \Omega = \widehat{\Omega} - \widehat{\Omega}\tr.
\]
In practice, we compute the Cayley retraction using the fixed-point iteration~\eqref{eq:cayley_fixed_point_map}. 
This yields an efficient and numerically stable retraction satisfying the standard first-order retraction conditions $ \Retraction_{x}(0_{x}) = x $ and $ \D \! \Retraction_x({0}_{x}) = \mathrm{Id} $ required for the convergence analysis of Riemannian optimization methods; see, e.g., \citep[\S 4.1]{AMS:2008}.

\subsubsection{Fixed-Point Approximation of the Cayley Retraction} \label[appendix]{app:cayley_approx_error}
In this section, we analyze the fixed-point iteration used to approximate the Cayley retraction and show that it geometrically converges to the exact Cayley transform. Most importantly, we further show that the resulting approximate mapping retains the retraction properties required by the convergence theory.

Let $ \{ Y_{i} \}_{i \geq 0}$ be the sequence of fixed-point iterates defined by
\begin{equation}\label{eq:fixed_point_iteration}
    Y_{i+1} = \mathcal F(Y_{i}),
\end{equation}
where $ \cF $ is the fixed-point map of~\eqref{eq:cayley_fixed_point_map}, namely, $ \cF(Y) \coloneqq X + \frac{\alpha}{2} \Omega (X + Y) $, and the initialization is $ Y_{0} = (I + \alpha \Omega) X $.
We recall the contraction property of a fixed-point map.
\begin{lemma}[Contraction of the fixed-point map] \label{lem:fp_contraction}
Assume that $ \alpha \|\Omega\|_2 < 2 $ ($ \alpha \geq 0 $). Then $ \cF $ is a contraction mapping on $\R^{n \times r}$ with contraction factor
\[
    \rho \coloneqq \frac{\alpha}{2} \| \Omega \|_2 < 1.
\]
\end{lemma}
\begin{proof}
For any $ Y_{1}, \, Y_{2} \in \R^{n \times r} $, 
\[
    \| \cF(Y_{1}) - \cF(Y_{2}) \| = \left\| \frac{\alpha}{2} \Omega Y_{1} - \frac{\alpha}{2} \Omega Y_{2} \right\| \leq \frac{\alpha}{2} \| \Omega \|_{2} \, \| Y_{1} - Y_{2} \| < \| Y_{1} - Y_{2} \|,
\]
which shows that $ \cF $ is a contraction mapping.
\end{proof}

\begin{theorem}[Geometric convergence to the Cayley retraction] \label{thm:fp_convergence}
    Let $ Y^{\star} \coloneqq \Retraction_X(\alpha \xi) $ be the exact Cayley retraction of $ \alpha \xi $ at $ X $.
    Under the assumptions of Lemma~\ref{lem:fp_contraction}, the fixed-point iterates defined by~\eqref{eq:fixed_point_iteration} satisfy
    \begin{equation}\label{eq:banach_fp_result}
        \forall i \geq 0, \qquad \| Y_{i} - Y^{\star} \| \leq \rho^{i} \, \| Y_{0} - Y^{\star} \|.
    \end{equation}
\end{theorem}
\begin{proof}
    Since by Lemma~\ref{lem:fp_contraction} $ \cF $ is a contraction mapping, the Banach fixed-point theorem applies, with the exact retraction $ Y^{\star} = \Retraction_{X}(\alpha \xi) $ being the unique fixed point of $ \cF $. The theorem's result~\eqref{eq:banach_fp_result} follows directly.
\end{proof}

We are now in the position to state the properties of the approximate retraction, which we denote by
\[
    \widetilde{\Retraction}_{X}(\alpha \xi) \coloneqq Y_{s},
\]
where $s$ is the total number of fixed-point iterations performed.

\begin{lemma}[Second-order accuracy] \label{lem:approx_retraction}
    For sufficiently small $\alpha$ and any fixed number of iterations $s$, the approximate Cayley retraction satisfies
    \[
        \widetilde{\Retraction}_X(\alpha \xi) = X + \alpha \xi + \cO(\alpha^{2}) + \cO( \rho^{s} ),
    \]
    and converges to the exact Cayley retraction as $ s \to \infty $.
\end{lemma}

In particular, for fixed $s$ and sufficiently small step size $\alpha$, the approximation error remains of higher order and does not dominate the first-order behavior of the update. This justifies using a finite number of fixed-point iterations in practice.

\begin{proof}
    The exact Cayley transform~\eqref{eq:exact_cayley_transform} admits the expansion
    \begin{equation}\label{eq:expansion_exact_retraction}
        \Retraction_{X}(\alpha \xi) = X + \alpha \xi + \cO(\alpha^{2}).        
    \end{equation}
    The initialization of the fixed-point iteration method is
    \[
        Y_{0} = (I + \alpha \Omega) X = X + \alpha \Omega X = X + \alpha \xi,
    \]
    which is clearly a first-order approximation of $ \Retraction_{X}(\alpha \xi) $. Therefore
    \begin{equation}\label{eq:initialization_minus_exact}
        \| Y_{0} - Y^{\star} \| = \cO(\alpha^{2}).
    \end{equation}
    By Theorem~\ref{thm:fp_convergence}, we have the contraction estimate \eqref{eq:banach_fp_result}. Inserting \eqref{eq:initialization_minus_exact} into the~\eqref{eq:banach_fp_result}, with $ i = s $, we obtain
    \[
        \| Y_{s} - Y^{\star} \| \leq \rho^{s} \, \cO(\alpha^{2}).
    \]
    Since $ \alpha $ is fixed within one update, the factor $ \cO(\alpha^{2}) $ can be absorbed into the constant, i.e.,
    \[
        \| Y_{s} - Y^{\star} \| \leq \cO(\rho^{s}).
    \]
    or, equivalently, $ Y_{s} = \Retraction_{X}(\alpha \xi) + \cO(\rho^{s}) $. Combining this with the expansion of the exact retraction~\eqref{eq:expansion_exact_retraction}, the result of the lemma follows immediately.
\end{proof}

In particular, Lemma~\ref{lem:approx_retraction} shows that the approximate Cayley mapping satisfies the first-order retraction conditions up to a controllable error. Such inexact retractions preserve the convergence guarantees of Riemannian first-order methods provided the approximation error is sufficiently small; see, e.g., standard analyses of inexact retraction schemes.

Lemma~\ref{lem:approx_retraction} shows that for any fixed $s$ and sufficiently small $\alpha$, the approximate Cayley transform defined by the fixed-point iteration still satisfies the first-order retraction condition. The term $ \cO(\rho^{s}) $ is the numerical approximation error, controlled by the number of iterations $s$ of the fixed-point approximation method. In particular, for a fixed $ \alpha $ and large enough $ s $, the approximate Cayley transform converges to the exact Cayley map.

\section{Proof of \texorpdfstring{\Cref{thm:convex_regretbound}}{Theorem 4.2}}\label[appendix]{app:regret_proof}
\begin{proof}
    Consider the sequence $\cL_t$ of convex loss functions and define
    \[
        R(T) \coloneqq \sum_{t=1}^T \cL_t(B_t,A_t)-\cL_t(B^*,A^*),
    \]
    where $(B_t,A_t)$ is the iteration in \Cref{algo:StiefelAdamW_BA}, and $(B^*,A^*)$ is a minimizer of $\sum_{t=1}^T\cL_t(B,A)$. Using the convexity of $\cL_t$, and defining $G_t^{A,B} = \nabla_{A,B} \cL_t(B_t,A_t)$, we get
    \begin{equation}\label{eq:regret_expansion}
        R(T) \leq \sum_{t=1}^T \langle B_t-B^*,G_t^B \rangle + \sum_{t=1}^T \langle A_t-A^*,G_t^A \rangle.
    \end{equation}
    We now bound the first term on the right-hand side of \eqref{eq:regret_expansion}, i.e., the Euclidean term $\langle B_t - B^*, G_t^B \rangle$. Let $H_t^B \colon \R^{m \times r} \to \R^{m\times r}$ the linear operator defined by $H_t^B(X) = (V_t^B+\varepsilon) \,\odot X$. We bound the norm $\|B_{t+1}-B^* \|_{(H_t^B)^{1/2}}^2$ (see \Cref{def:operator_induced_norm}):
    \begin{align}\label{eq:bound_iterate_norm_B}
        \|B_{t+1}-B^* \|_{(H_t^B)^{1/2}}^2 =& \| B_t - \eta_t (H_t^B)^{-1/2}M_t^B - B^* \|_{(H_t^B)^{1/2}}^2 = \|B_t-B^* \|_{(H_t^B)^{1/2}}^2 + \\ &+ \nonumber\eta_t^2 \|(H_t^B)^{-1/2}M_t^B \|_{(H_t^B)^{1/2}}^2  -2 \eta_t \langle B_t-B^*, (H_t^B)^{-1/2}M_t^B\rangle_{(H_t^B)^{1/2}} \\ \nonumber
        \underset{\text{\Cref{lemma:operator_induced_norm}}}{=}&   \|B_t-B^* \|_{(H_t^B)^{1/2}}^2 + \eta_t^2 \|M_t^B \|_{(H_t^B)^{-1/2}}^2 - 2 \eta_t \langle B_t-B^*,  M_t^B \rangle,
    \end{align}
    where the last inner product is in the Frobenius norm (for clarity, we always omit the subscript).
    By rearranging \eqref{eq:bound_iterate_norm_B} (bringing the inner product on the left-hand side and the norm on the right-hand side), we get:
    \begin{align}\label{eq:bound_inner_prod_B}
        2 \eta_t \langle B_t-B^*,  M_t^B \rangle = \underbrace{\|B_t-B^* \|_{(H_t^B)^{1/2}}^2 -\|B_{t+1}-B^* \|_{(H_t^B)^{1/2}}^2}_{\eqqcolon \Delta^{B}_{t}} +\eta_t^2 \|M_t^B \|_{(H_t^B)^{-1/2}}^2.
    \end{align}
    Using the definition of $M_t^B = \beta_1 M_{t-1}^B + (1-\beta_1)G_t^B$ in \eqref{eq:bound_inner_prod_B}, and by defining $\Delta_t^B\coloneqq\|B_t-B^* \|_{(H_t^B)^{1/2}}^2 -\|B_{t+1}-B^* \|_{(H_t^B)^{1/2}}^2$, we get 
    \begin{align}\label{eq:bound_inner_prod_B_grad_final}
         \langle B_t-B^*,  G_t^B \rangle =& \frac{1}{2(1-\beta_1)\eta_t} \Delta_t^B +\frac{\eta_t}{2(1-\beta_1)} \|M_t^B \|_{(H_t^B)^{-1/2}}^2 - \frac{\beta_1}{1-\beta_1} \langle B_t-B^*,  M_{t-1}^B \rangle \\ \nonumber\leq &\frac{1}{2(1-\beta_1)\eta_t} \Delta_t^B +\frac{\eta_t}{2(1-\beta_1)} \|M_t^B \|_{(H_t^B)^{-1/2}}^2 + \frac{\beta_1}{1-\beta_1} \left\lvert \langle B_t-B^*,  M_{t-1}^B \rangle \right\rvert \\ \underset{\text{\Cref{lemma:young_inequality}}}{\leq} & \nonumber \frac{\Delta_t^B}{2(1-\beta_1)\eta_t}  +\frac{\eta_t}{2(1-\beta_1)} \|M_t^B \|_{(H_t^B)^{-1/2}}^2 + \frac{\beta_1}{2(1-\beta_1)\alpha_t^2}\|M_{t-1}^B \|_{(H_t^B)^{-1/2}}^2 \\ +& \nonumber \frac{\beta_1 \alpha_t^2}{2(1-\beta_1)}\|B_t-B^* \|_{(H_t^B)^{1/2}}^2.
    \end{align}
    Apart from the retraction, the term $\langle A_t-A^*,G_t^A\rangle$ is similar to \eqref{eq:bound_inner_prod_B_grad_final}, i.e.,
    \[
        A_{t+1} = \Retraction_{A_t}\Bigl(- \eta_t \Proj_{A_t} \! \left( (H_t^A)^{-1/2}M_t^{A} \right) \Bigr).
    \]
    We define 
    \[
        E_t\coloneqq\Retraction_{A_t}(-\eta_t \Proj_{A_t}(H_t^A)^{-1/2}M_t^A) - (A_t - \eta_t \Proj_{A_t}(H_t^A)^{-1/2}M_t^A),
    \]
    and the invertible linear operator $\Gamma_t = \Proj_{A_t} (H_t^A)^{-1/2} \Proj_{A_t} \colon \T_{A_t} \Stnr \to \T_{A_t} \Stnr$. We define $\bar \Gamma_t$ as an extension of the previous map on the whole space, $\bar \Gamma_t = \Gamma_t + \gamma_t(I-\Proj_{A_t})$, where $\gamma_t>0$ is a scalar. With a small abuse of notation, we will still denote by $\Gamma_t$ the map $\bar \Gamma_t$ when there is no risk of confusion. Let $ D_t \coloneqq \Proj_{A_t} (H_t^A)^{-1/2} M_t^A \in \T_{A_t}\Stnr $, and consider the norm
    \begin{equation*}
        \adjustbox{max width=\linewidth}{$\displaystyle
            \begin{aligned}
                \|A_{t+1}-A^* \|_{\Gamma_t^{-1}}^2 &= \|A_{t} -\eta_t D_t + E_t-A^* \|_{\Gamma_t^{-1}}^2 \\
                &= \|A_{t} -A^* \|_{\Gamma_t^{-1}}^2-2\eta_t \langle \Gamma_t^{-1} D_t,A_t-A^* \rangle + \eta_t^2\| D_t\|_{\Gamma_t^{-1}}^2 \\
                &\quad -2\eta_t \langle E_t, \Gamma_t^{-1}D_t \rangle + 2 \langle \Gamma_{t}^{-1} E_t,A_t-A^* \rangle + \| E_t \|_{\Gamma_t^{-1}}^2 \\
                \underset{\langle E_t, A_t-A^* \rangle_{\Gamma_{t}^{-1}} \leq 0}{\leq} & \|A_{t} -A^* \|_{\Gamma_t^{-1}}^2-2\eta_t \langle \Gamma_t^{-1} D_t,A_t-A^* \rangle+ \eta_t^2\| D_t\|_{\Gamma_t^{-1}}^2 + \underset{\eqqcolon\delta_t}{\underbrace{-2\eta_t \langle E_t, \Gamma_t^{-1}D_t \rangle + \|E_t \|_{\Gamma_t^{-1}}^2}}.
            \end{aligned}
        $}
    \end{equation*}
    namely,
    \begin{equation}\label{eq:A_norm_expansion}
        \| A_{t+1} - A^{\ast} \|_{\Gamma_{t}^{-1}}^{2} \leq \| A_{t} - A^{\ast} \|_{\Gamma_{t}^{-1}}^{2} - 2 \eta_{t} \left\langle \Gamma_{t}^{-1} D_{t}, A_{t} - A^{\ast} \right\rangle + \eta_{t}^{2} \| D_{t} \|_{\Gamma_{t}^{-1}}^{2} + \delta_{t}.
    \end{equation}
    Bringing the inner product in~\eqref{eq:A_norm_expansion} to the left-hand side and the norm to the right-hand side leads to
    \[
        2 \eta_{t} \left\langle \Gamma_{t}^{-1} D_{t}, A_{t} - A^{\ast} \right\rangle \leq \| A_{t} - A^{\ast} \|_{\Gamma_{t}^{-1}}^{2} - \| A_{t+1} - A^{\ast} \|_{\Gamma_{t}^{-1}}^{2} + \eta_{t}^{2} \| D_{t} \|_{\Gamma_{t}^{-1}}^{2} + \delta_{t}.
    \]
    Defining $ \Delta_t^A\coloneqq \|A_t-A^* \|_{\Gamma_t^{-1}}^2-\|A_{t+1}-A^* \|_{\Gamma_t^{-1}}^2 $, and dividing by $ 2\eta_{t} $, we can write
    \[
        \langle \Gamma_t^{-1} D_t,A_t-A^* \rangle \leq \frac{1}{2\eta_t} \Delta_{t}^{A} + \frac{\eta_t}{2}\|D_t \|_{\Gamma_t^{-1}}^2+ \frac{1}{2\eta_t}\delta_{t}.
    \]
    We now notice that, by definition of $\Gamma_t$, we have $\Gamma_t^{-1} D_t= M_t^A$, and, by using again Young inequality (\Cref{lemma:young_inequality}), we get
    \begin{equation}\label{eq:bound_inner_prod_A_grad_final}
        \adjustbox{max width=.91\linewidth}{$\displaystyle
            \begin{aligned}
                \langle G_t^A,A_t-A^* \rangle &\leq \frac{1}{2(1-\beta_1)\eta_t}\Delta_t^A + \frac{\eta_t}{2(1-\beta_1)}\|D_t \|_{\Gamma_t^{-1}}^2 + \frac{1}{2(1-\beta_1)\eta_t}\delta_t \\
                &\quad + \frac{\beta_1 \alpha_t^2}{2(1-\beta_1)} \|A_t-A^{\ast} \|_{\mathrm{F}}^2 + \frac{\beta_1}{2(1-\beta_1) \alpha_t^2} \|M_{t-1}^A \|_{\mathrm{F}}^2 \\
                &= \frac{1}{2(1-\beta_1)\eta_t}\Delta_t^A + \frac{\eta_t}{2(1-\beta_1)} \left\|(H_t^{A})^{-1/4}(H_t^A)^{1/4}\Gamma_t^{1/2}\Gamma_t^{-1}D_t \right\|_{\mathrm{F}}^2 + \frac{1}{2(1-\beta_1)\eta_t}\delta_t \\
                &\quad + \frac{\beta_1 \alpha_t^2}{2(1-\beta_1)} \left\|(H_t^{A})^{-1/4}(H_t^A)^{1/4}(A_t-A^{\ast}) \right\|_{\mathrm{F}}^2 + \frac{\beta_1}{2(1-\beta_1) \alpha_t^2} \|(H_t^{A})^{-1/4}(H_t^A)^{1/4}M_{t-1}^A \|_{\mathrm{F}}^2 \\
                &= \frac{1}{2(1-\beta_1)\eta_t}\Delta_t^A + \frac{\eta_t}{2(1-\beta_1)}\|(H_t^A)^{1/4}\Gamma_t^{1/2}M_t^A\|_{(H_t^{A})^{-1/2}}^2 + \frac{1}{2(1-\beta_1)\eta_t}\delta_t \\
                &\quad + \frac{\beta_1 \alpha_t^2}{2(1-\beta_1)} \|(H_t^A)^{1/4}(A_t-A^{\ast}) \|_{(H_t^{A})^{-1/2}}^2 + \frac{\beta_1}{2(1-\beta_1) \alpha_t^2} \|(H_t^{A})^{-1/4}M_{t-1}^A \|_{(H_t^{A})^{1/2}}^2 \\
                &\leq \frac{\Delta_t^A}{2(1-\beta_1)\eta_t} + \frac{\eta_{t} \|(H_t^A)^{1/4}\Gamma_t^{1/2} \|_{\mathrm{op}}^2}{2(1-\beta_1)}\|M_t^A\|_{(H_t^{A})^{-1/2}}^2 + \frac{\delta_t}{2(1-\beta_1)\eta_t} \\
                &\quad + \frac{\beta_1 \alpha_t^2 \|(H_t^A)^{1/4}\|_{\mathrm{op}}^2}{2(1-\beta_1)} \| A_t - A^{\ast} \|_{(H_t^{A})^{-1/2}}^2 + \frac{\beta_1 \|(H_t^{A})^{-1/4} \|_{\mathrm{op}}^2}{2(1-\beta_1) \alpha_t^2} \|M_{t-1}^A \|_{(H_t^{A})^{1/2}}^2.
            \end{aligned}
        $}
    \end{equation}
    The term $\delta_t$ is the only one structurally different from the ones in \eqref{eq:bound_inner_prod_B_grad_final}. Thus, \eqref{eq:bound_inner_prod_A_grad_final} can be bounded using the definition of $E_t$ (retraction error) with Lagrange remainder error
    \[
        \|E_t \| \leq C\| \D^2 \Retraction_{A_t}(\zeta) \| \|\eta_t D_t \|^2 \leq \tilde C \eta_t^2 \|D_t \|^2,
    \]
    as
    \begin{align*}
        \frac{\delta_t}{2(1-\beta_1)\eta_t} &\leq \frac{1}{2(1-\beta_1)\eta_t}\Big[2 \eta_t \|E_t \|\|M_t^A \| + \| \Gamma_t^{-1}\|_{\mathrm{op}} \|E_t \|^2\Big] \\
        &\leq \frac{1}{2(1-\beta_1)}\Big[2  \|E_t \|\|M_t^A \| + \frac{1}{\eta_t}\| \Gamma_t^{-1}\|_{\mathrm{op}} \|E_t \|^2\Big] \\
        &\leq \frac{1}{2(1-\beta_1)}\Big[2 \eta_t^2  \|D_t \|^{2} \|M_t^A \| + \eta_t^3\| \Gamma_t^{-1}\|_{\mathrm{op}} \|D_t \|^4\Big].
    \end{align*}
    Using the definition
    \[
        \| M_t^A \| \coloneqq \|(1-\beta_1) \sum_{s=1}^t\beta_1^{t-s} G_s^{A}\| \leq G_{\infty},
    \]
    and the fact that $\|D_t \| \lesssim G_{\infty}$, $\|\Gamma_t^{-1}\|_{\mathrm{op}}$ bounded, we get
    \begin{equation}\label{eq:bound_delta}
        \sum_{t=1}^T \frac{\delta_t}{2(1-\beta_1)\eta_t} \lesssim \frac{1}{2(1-\beta_1)}\Big[ G_{\infty}^{3} \sum_{t=1}^T \eta_t^2 + G_{\infty}^{4} \sum_{t=1}^T \eta_t^3\Big].
    \end{equation}
    By combining equations \eqref{eq:regret_expansion}, \eqref{eq:bound_inner_prod_B_grad_final}, \eqref{eq:bound_inner_prod_A_grad_final} and \eqref{eq:bound_delta}, we get
    \begin{equation}\label{eq:final_regret_bound_full}
        \adjustbox{max width=.91\linewidth}{$\displaystyle
            \begin{aligned}
                R(T) &\leq \sum_{t=1}^T \Biggl( \frac{\Delta_t^B}{2(1-\beta_1)\eta_t} + \frac{\eta_t \|M_t^B \|_{(H_t^B)^{-1/2}}^2}{2(1-\beta_1)} + \frac{\beta_1 \|M_{t-1}^B \|_{(H_t^B)^{-1/2}}^2}{2(1-\beta_1)\alpha_t^2} + \frac{\beta_1 \alpha_t^2 \|B_t-B^* \|_{(H_t^B)^{1/2}}^2}{2(1-\beta_1)} \Biggr) \\
                &\quad + \sum_{t=1}^T \Biggl( \frac{\Delta_t^A}{2(1-\beta_1)\eta_t} + \frac{\eta_t K_1 \|M_t^A \|_{(H_t^A)^{-1/2}}^2}{2(1-\beta_1)} + \frac{\beta_1 K_2 \|M_{t-1}^A \|_{(H_t^A)^{1/2}}^2}{2(1-\beta_1)\alpha_t^2} + \frac{\beta_1 \alpha_t^2 K_3 \|A_t-A^* \|_{(H_t^A)^{-1/2}}^2}{2(1-\beta_1)} \Biggr) \\
                &\quad + \frac{1}{2(1-\beta_1)} \Biggl[ G_{\infty}^{3} \sum_{t=1}^T \eta_t^2 + G_{\infty}^{4} \sum_{t=1}^T \eta_t^3 \Biggr],
            \end{aligned}
        $}
    \end{equation}
    where $K_1,K_2,K_3$ are three constants bounding the operator norms in \eqref{eq:bound_inner_prod_A_grad_final}. Notice that up to these constants, the terms in $A$ and $B$ are similar to each other; therefore, we focus on bounding the terms in $B$ (the ones in $A$ yield analogous bounds).
    
    Now, we observe that, thanks to \citep[Lemma 2]{convergence_adam_beyond}, for $\eta_t = \eta/\sqrt{t}$ we obtain the following bounds
    \[
        \sum_{t=1}^T \eta_t \|M_t^B \|_{(H_t^B)^{-1/2}}^2 \leq \frac{\eta \|G^B \|_{L^1L^2([0,T])} \sqrt{1 + \log T}}{(1-\beta_1)(1-\beta_1/\sqrt{\beta_2})\sqrt{1-\beta_2}},
    \]
    where $ \|G^B \|_{L^1 L^2([0,T])}$ is the $L^1-L^2$ norm defined by
    \[
        \|G^B \|_{L^1L^2([0,T])}\coloneqq \sum_{i,j} \left( \sum_{t=1}^T \left|(G_t^B)_{ij}\right|^2 \right)^{1/2}.
    \]
    Moreover,
    \[
        \sum_{t=1}^T \eta_t \|M_t^A \|^2_{(H_t^A)^{-1/2}} \leq \frac{\eta \|G^A \|_{L^1L^2([0,T])} \sqrt{1+\log T}}{(1-\beta_1)(1-\beta_1/\sqrt{\beta_2})\sqrt{1-\beta_2}},
    \]
    where
    \begin{small}
        \begin{equation}\label{eq:bound_momentum_series}
            \sum_{t=1}^T \eta_t^2 = \eta^{2} \sum_{t=1}^T \frac{1}{t} \leq \eta^{2} (1 + \log T),
            \quad \text{and} \quad
            \sum_{t=1}^T \eta_t^3 = \sum_{t=1}^T \frac{\eta}{t^{3/2}} \leq 1+\int_{1}^T t^{-3/2} \, \mathrm{d}t = 3 - \frac{2}{\sqrt{T}}.        
    \end{equation}
    \end{small}
        
    Similarly, we can bound the third term on the right-hand side of~\eqref{eq:final_regret_bound_full}, i.e.,
    \begin{align}\label{eq:bound_time_diff}
        \| M_{t-1}^B \|_{(H_t^B)^{-1/2}}^2 &= \left\|(H_t^B)^{-1/4}(H_{t-1}^B)^{1/4}(H_{t-1}^B)^{-1/4}(M_{t-1}^B) \right\|_{\mathrm{F}}^2 \nonumber \\
        & \leq \left\|(H_t^B)^{-1/4}(H_{t-1}^B)^{1/4} \right\|_{2 \to 2}^{2} \|M_{t-1}^B \|_{(H_{t-1}^B)^{-1/2}}^2 \nonumber \\
        & \leq \Bigl(\frac{1}{\beta_2}\Bigr)^{1/4} \|M_{t-1}^B \|_{(H_{t-1}^B)^{-1/2}}^2,
    \end{align}
    where the last inequality is given by the fact that 
    \begin{align*}
        \left\| (H_t^B)^{-1/4}(H_{t-1}^B)^{1/4} \right\|_{2 \to 2}^4 =& \left\| \Bigl(\frac{V_{t-1}^B}{V_t^B}\Bigr)^{1/4} \right\|_{L^{\infty}}^4 = \left\| \frac{V_{t-1}^B}{\beta_2 V_{t-1}^B + (1-\beta_2)\, (G_t^B)^{\circ 2}} \right\|_{L^{\infty}} \\
        =& \nonumber\left\| 1+\frac{(1-\beta_2)V_{t-1}^B}{\beta_2 V_{t-1}^B + (1-\beta_2) \, (G_t^B)^{\circ 2}} \right\|_{L^{\infty}} \leq 1+ \frac{1-\beta_2}{\beta_2}.
    \end{align*}
    The result in \eqref{eq:bound_time_diff} allows to apply \citep[Lemma 2]{convergence_adam_beyond} (or, equivalently, using \eqref{eq:bound_momentum_series}) also on the time-shifted term $\|M_{t-1}^B \|_{(H_{t}^B)^{-1/2}}$, which for $1/\alpha_t^2 =\alpha/\sqrt{t}$ gives
    \begin{align}\label{eq:bound_shifted_momentum_B}
        \sum_{t=1}^T \frac{\|M_{t-1}^B \|_{(H_{t}^B)^{-1/2}}^2}{\alpha_t^2} &\leq \sum_{t=1}^T \frac{\alpha}{\sqrt t}\Bigl(\frac{1}{\beta_2}\Bigr)^{1/4}\|M_{t-1}^B \|_{(H_{t-1}^B)^{-1/2}}^2 \nonumber \\
        &\leq \frac{\alpha \|G^B \|_{L^1L^2([0,T])} \sqrt{1+\log T}}{\beta_2^{1/4}(1-\beta_1)(1-\beta_1/\sqrt{\beta_2})\sqrt{1-\beta_2}}.
    \end{align}
    Using assumptions (H2) that $ \beta_{1,t} = \beta_1 b^t $, and (H3) that $\sup_{t} \|B_t -B^*\|_{\max} \leq D_{\infty}$ (the same quantity can be bounded by $2$ for the term in $A$, given that Stiefel is compact and has finite diameter), together with the last bound from \citep[Theorem 4]{convergence_adam_beyond}, and combining the results from \eqref{eq:bound_momentum_series}, \eqref{eq:bound_shifted_momentum_B},\eqref{eq:final_regret_bound_full} (for $A_t$ they are similar), we get
    \begin{equation}\label{eq:bound_regret_reduced}
        \begin{split}
            R(T) &\leq \frac{D_{\infty}^2 \bigl\|(H_T^B)^{1/4}\bigr\|_{2\to 2}^2}{2\eta_T(1-\beta_1)} + \frac{D_{\infty}^2}{2(1-\beta_1)} \sum_{t=1}^T \beta_{1,t}\alpha_t^2 \bigl\|(H_t^B)^{1/4}\bigr\|_{2\to 2}^2 \\
            &\quad + \frac{\mathrm{diam}_{\infty}(\Stnr)^2 \bigl\|(H_T^A)^{1/4}\bigr\|_{2\to 2}^2}{2\eta_T(1-\beta_1)} + \frac{\mathrm{diam}_{\infty}(\Stnr)^2}{2(1-\beta_1)} \sum_{t=1}^T \beta_{1,t}\alpha_t^2 \bigl\|(H_t^A)^{1/4}\bigr\|_{2\to 2}^2 \\
            &\quad + \frac{(\eta+2\alpha\beta_1^{3/4})\sqrt{1+\log T}}{2(1-\beta_1)^2(1-\beta_1/\sqrt{\beta_2})\sqrt{1-\beta_2}} \bigl\|G^A\bigr\|_{L^1L^2([0,T])} \\
            &\quad + \frac{(\eta+2\alpha\beta_1^{3/4})\sqrt{1+\log T}}{2(1-\beta_1)^2(1-\beta_1/\sqrt{\beta_2})\sqrt{1-\beta_2}} \bigl\|G^B\bigr\|_{L^1L^2([0,T])} \\
            &\quad + \frac{1}{2(1-\beta_1)} \left[ G_\infty^3 \eta^2 (1+\log T) + G_\infty^4 \eta^3 \left(3-\frac{2}{\sqrt{T}}\right) \right].
        \end{split}
        \raisetag{3\baselineskip}
    \end{equation}
    To conclude, we bound the two terms $\sum_{t=1}^T \beta_{1,t} \alpha_t^2 \|(H_t^B)^{1/4}\|_{2\to 2}^2$ uniformly in $t$, for the choice $\beta_{1,t} = \beta_1 b^t$, and $\alpha_t^2 = \sqrt{t}/{\alpha}$, we get the upper bound
    \begin{align*}
        \sum_{t=1}^T \beta_{1,t} \alpha_t^2 \|(H_t^B)^{1/4}\|_{2\to 2}^2 &\leq \frac{\beta_1}{\alpha}\| \| (H^B)^{1/4}\|_{2 \to 2}^2 \|_{L^\infty_t} \sum_{t=1}^T \sqrt{t} b^t \\
        &\underset{\substack{\text{Cauchy--Schwarz} \\ \sqrt{t}b^{t/2},~b^{t/2}}}{\leq} \frac{\beta_1}{\alpha}\| \| (H^B)^{1/4}\|_{2 \to 2}^2 \|_{L^\infty_t} \left(\sum_{t=1}^T t b^t\right)^{1/2} \left(\sum_{t=1}^T b^{t/2}\right)^{1/2} \\
        &\leq \frac{\beta_1}{\alpha}\| \| (H^B)^{1/4}\|_{2 \to 2}^2 \|_{L^\infty_t} \frac{b}{ (1-b)^{3/2}}.
    \end{align*}
    This gives the final bound
    \begin{equation*}
        \begin{aligned}
            R(T) &\lesssim \frac{D_{\infty}^2 \|(H_T^B)^{1/4} \|_{2 \to 2}^2}{2\eta_T(1-\beta_1)} + \frac{D_{\infty}^2}{2(1-\beta_1)}\frac{\beta_1}{\alpha}\| \| (H^B)^{1/4}\|_{2 \to 2}^2 \|_{L^\infty_t}\frac{b}{(1-b)^{3/2}} \\
            &\quad + \frac{(\eta+2\alpha \beta_1^{3/4}) \sqrt{1+\log T}}{2(1-\beta_1)^2(1-\beta_1/\sqrt{\beta_2})\sqrt{1-\beta_2}} \|G^B \|_{L^1L^2([0,T])} + \frac{\mathrm{diam}_{\infty}(\Stnr)^2 \|(H_T^A)^{1/4} \|_{2 \to 2}^2}{2\eta_T(1-\beta_1)} \\
            &\quad + \frac{\mathrm{diam}_{\infty}(\Stnr)^2}{2(1-\beta_1)}\frac{\beta_1}{\alpha}\| \| (H^A)^{1/4}\|_{2 \to 2}^2 \|_{L^\infty_t} \frac{b}{(1-b)^{3/2}} \\
            &\quad + \frac{(\eta+2\alpha \beta_1^{3/4}) \sqrt{1+\log T}}{2(1-\beta_1)^2(1-\beta_1/\sqrt{\beta_2})\sqrt{1-\beta_2}} \|G^A \|_{L^1L^2([0,T])} \\
            &\quad + \frac{1}{2(1-\beta_1)} \left[ G_\infty^{3} \eta^{2} (1+\log T) + G_{\infty}^{4} \eta^{3} \left(3 - \frac{2}{\sqrt{T}}\right) \right].
        \end{aligned}
    \end{equation*}
    By using the fact that $\eta_T^{-1}$, we can collect the constant $C_1$ of the terms of order $1$, the constant $C_2$ of order $\sqrt{T}$, the constant $C_3$ of the terms of order $\sqrt{1+\log T}$, and $C_4$ of the terms of order $\log T$ and with $C_5$ the constant of the term of order $T^{-1/2}$ to get the final bound. In particular, the order of the bound is
    \[
        R(T) \leq C_1\sqrt{T}+ C_2+ C_3\sqrt{1+\log T}+ C_4\log T + C_5 \, T^{-1/2}.
    \]
\end{proof}

\begin{definition}($H$-norm induced by a full-rank operator)
\label[definition]{def:operator_induced_norm}
    Consider a self-adjoint positive definite operator $H \colon V \to V$ defined on a finite-dimensional real Hilbert space $(V,g)$. We define the $H$-weighted inner product as
    \[
        g_H(v,w) \coloneqq g(H^{1/2} v,H^{1/2}w),
    \]
    and denote the corresponding norm as
    \[
        \| x \|_{H}= \|H^{1/2} x\|_{g}.
    \]
\end{definition}

\begin{lemma}(Properties of $H$-induced inner products)
\label[lemma]{lemma:operator_induced_norm}
    Let $H$ and $V$ be as in \Cref{def:operator_induced_norm}. Then the following holds:
    \begin{itemize}
        \item $g_{H^\gamma}(H^\alpha v,H^\beta w) = g_{H^\gamma}(v,H^{\alpha+\beta}w) = g_{H^{\alpha + \beta + \gamma}}(v,w)$,
        \item $\| H^\alpha x \|_{H^\gamma} = \|H^{\alpha+\gamma/2} x\|_g = \|x \|_{H^{2\alpha + \gamma}}$.
    \end{itemize}
\end{lemma}

\begin{proof}
    The first point follows from the definition, self-adjointness $g(Hv,w) = g(v,Hw)$ (which holds for any power), and the fact that powers commute ($H^\alpha H^\beta = H^\beta H^\alpha$ for all $\alpha,\beta$):
    \begin{align*}
        g_{H^\gamma}(H^\alpha v,H^\beta w) =& g(H^{\alpha + \gamma/2}v,H^{\beta + \gamma/2}w) = g(v,H^{\alpha+\beta+\gamma}w) = g(H^{(\alpha+\beta+\gamma)/2}v,H^{(\alpha+\beta+\gamma)/2}w) \\
        =& g_{H^{\alpha+\beta+\gamma}}(v,w).
    \end{align*}
    The first equality is similar
    \begin{align*}
        g_{H^\gamma}(H^\alpha v,H^\beta w) =& g(H^{\alpha }H^{\gamma/2}v,H^{\beta} H^{ \gamma/2}w)  = g(H^{\gamma/2}v, H^{ \gamma/2}H^{\alpha+\beta} w) = g_{H^\gamma}(v,H^{\alpha+\beta}w).
    \end{align*}
    The norm equality follows immediately from the definition
    \[
    \|H^\alpha x \|_{H^\gamma}^2 = g_{H^\gamma}(H^\alpha x ,H^\alpha x) =  g(H^{\alpha+\gamma/2} x ,H^{\alpha+\gamma/2} x) = \|H^{\alpha+\gamma/2} x \|_g^2 =\|x \|_{H^{2 \alpha + \gamma}}^2.
     \]
\end{proof}

\begin{lemma}(Young inequality for dual $H$ norms)
\label[lemma]{lemma:young_inequality}
    Let $H$, $V$ be as in \Cref{def:operator_induced_norm} and let $ u, v \in V $, $\zeta \ne 0$, and $\alpha \in \R$. Then, we have
    \[
        g(u,v) \leq \frac{1}{2\zeta^2} \|u \|^2_{H^{-\alpha}} + \frac{\zeta^2}{2}\|v \|_{H^{\alpha}}^2.
    \]
\end{lemma}

\begin{proof}
    Consider the expansion
    \[
        0\leq \|u-v \|_g^2 = g(u-v,u-v) = \|u \|_g^2 + \| v\|_g^2 - 2 g(u,v),
    \]
    which implies
    \[
        g(u,v) \leq \frac{1}{2}\|u \|_g^2 + \frac{1}{2}\| v\|_g^2.
    \]
    Now, since $H$ is self-adjoint and positive definite with respect to the inner product $g$, we have that $\zeta H^{\alpha}$ is too, and therefore we get
    \[
        g(\zeta^{-1} H^{-\alpha/2}u,\zeta H^{\alpha/2} v) \leq \frac{1}{2}\|\zeta^{-1} H^{-\alpha/2}u \|_g^2 + \frac{1}{2}\| \zeta H^{\alpha/2} v\|_g^2 = \frac{1}{2\zeta^2}\|u \|_{H^{-\alpha}}^2 + \frac{\zeta^2}{2}\| v\|_{H^\alpha}^2.
    \]
\end{proof}

\subsection{Additional Results on Different Numerical Retractions}\label[appendix]{app:compare_retractions}
In this section, we present numerical results comparing several possible retraction choices on the Stiefel manifold. In \Cref{fig:benchmark_retractions}, we compare matrix size against GPU wall-clock time and final error. In particular, we compare QR decomposition, polar decomposition, a direct solver for the Cayley linear system (Cayley-Direct), an iterative method for the Cayley linear system that employs the Sherman--Morrison--Woodbury formula (Cayley-SMW), the fixed-point iteration (Cayley-FP), and the Newton--Schulz iteration.
\begin{figure}[ht]
    \centering
    \includegraphics[width=\linewidth]{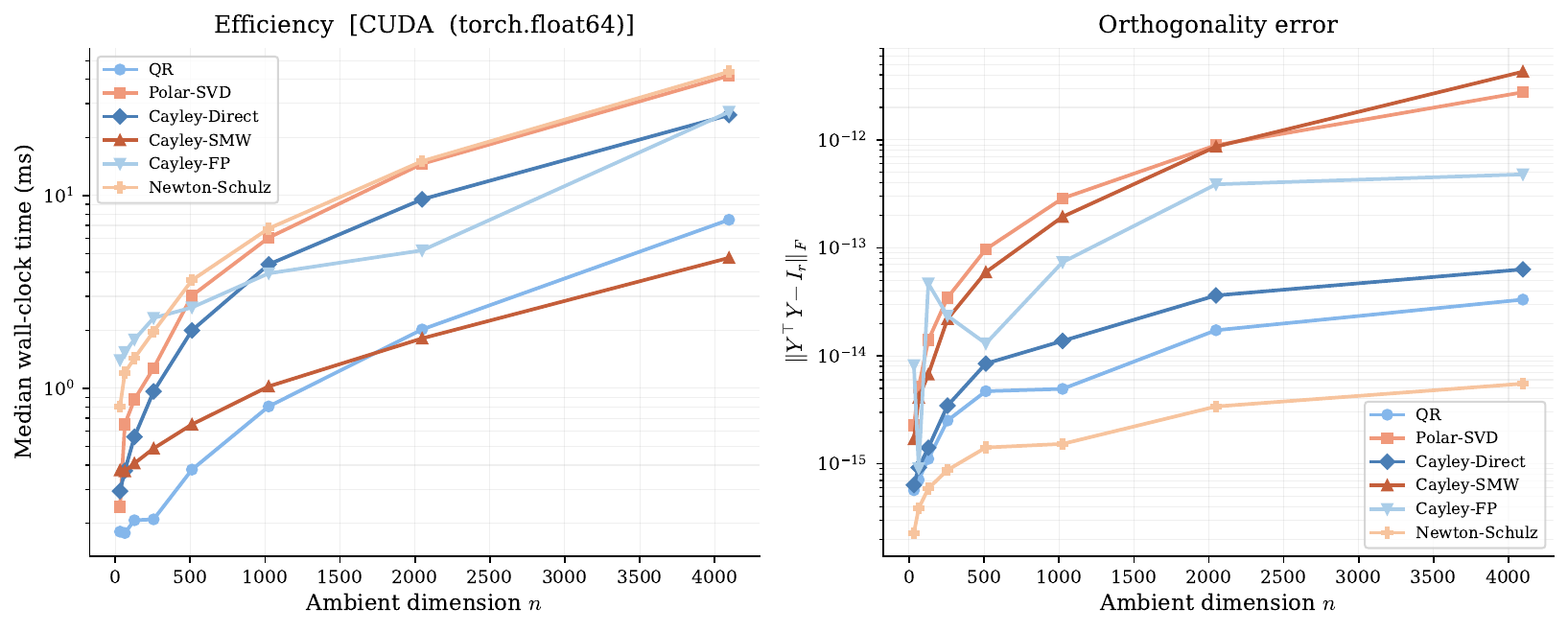}
    \caption{Comparison of different numerical retraction methods.}
    \label{fig:benchmark_retractions}
\end{figure}

\section{Proof of \texorpdfstring{\Cref{prop:bounded_gradient}}{Proposition \ref*{prop:bounded_gradient}}}\label[appendix]{app:proof_bounded_gradient}
Let $(B,A)\in \mathcal F$ and consider
\begin{align*}
    \| \nabla (\cL \circ \Phi)(tB,t^{-1}A)\|^2 &= \| \D\Phi(tB,t^{-1}A)\tr \nabla \cL(W)\|^2 \\
    &= \|tB\tr \nabla \cL(W)\|^2 + \|\nabla \cL(W) A\tr t^{-1} \|^2 \underset{{t \to 0}}{\longrightarrow}+\infty,
\end{align*}
and therefore $\|\nabla (\cL \circ \Phi) \|_{L^\infty(\mathcal F)} = + \infty$. \\
For the second claim, fix $A$ and $B$ such that $\tilde \Phi(B, A) = W$. Then, we have
\[
    \sup_{(B',A') \in  \widetilde{\mathcal F}}\| \nabla (\cL \circ \widetilde \Phi)(B',A')\|^2 = \sup_{O \in \mathrm{St}(r,r)} \| \nabla (\cL \circ \widetilde \Phi)(BO,O\tr A)\|^2 < + \infty,
\]
because of compactness of $\mathrm{St}(r,r)$ and continuity.

\section{Additional Experimental Details}

\subsection{GPT2 E2E Fine-Tuning}
We use hyperparameters tuned as in \citep{ZhangPilanci2024}, as reported in \Cref{tab:GPT_finetuning_hyperparams}. 
\begin{table}[t]
    \centering
    \small
    \begin{tabularx}{\linewidth}{l|*{4}{>{\centering\arraybackslash}X}}
        \toprule
        \textbf{Configuration} & \algname{} & AdamW & Scaled AdamW & GeoLoRA \\
        \midrule
        Learning rate & $8 \times 10^{-3}$ & $8 \times 10^{-3}$ & $8 \times 10^{-3}$ & $8 \times 10^{-3}$ \\
        Weight decay & $10^{-4}$ & $10^{-2}$ & $10^{-2}$ & $10^{-4}$ \\
        Learning rate schedule & Linear & Linear & Linear & Linear  \\
        $(\beta_1,\beta_2)$ & $(0.98,0.98)$ & $(0.9,0.999)$ & $(0.7,0.8)$ & $(0.98,0.98)$ \\
        \bottomrule
    \end{tabularx}
    \caption{Hyperparameters for GPT2 LoRA fine-tuning on E2E.}
    \label{tab:GPT_finetuning_hyperparams}
\end{table}

\end{document}